\documentclass[twoside]{article}
\usepackage[dvipsnames]{xcolor}
\usepackage{graphicx}
\usepackage{amsthm}
\usepackage{amsmath}
\usepackage{mathtools}
\usepackage{mathrsfs}
\usepackage{amssymb}
\usepackage{lipsum}
\usepackage{etoolbox}
\usepackage[linesnumbered,lined,boxed,ruled]{algorithm2e}
\usepackage[section]{placeins}
\usepackage[breaklinks,colorlinks,citecolor=blue,linkcolor=blue,urlcolor=blue]{hyperref}
\usepackage[all]{hypcap}
\usepackage{cancel}
\usepackage{csquotes}
\usepackage[toc,page]{appendix}
\usepackage{multirow}
\usepackage{datetime}
\usepackage{times}
\usepackage[format=plain, labelfont={bf,it}, textfont=it]{caption}
\usepackage{tcolorbox}
\usepackage{natbib, bibentry}
\usepackage{orcidlink}
\usepackage{fancyhdr}
\usepackage{enumitem}

\usepackage{scalerel}
\usepackage{stackengine,wasysym}

\newdateformat{mdydate}{\monthname[\THEMONTH] \THEDAY, \THEYEAR}
\newdateformat{monthyeardate}{\monthname[\THEMONTH] \THEYEAR}
\newdateformat{vmonthyeardate}{\THEYEAR.\THEMONTH.\THEDAY}

\newcommand{\be}{\begin{equation}}
\newcommand{\ee}{\end{equation}}
\newcommand{\bea}{\begin{eqnarray}}
\newcommand{\eea}{\end{eqnarray}}

\usepackage{geometry}
\graphicspath{{figures}}

\newcommand{\volume}{3}
\newcommand{\firstpage}{1}
\newcommand{\lastpage}{32}
\newcommand{\yyyy}{2026}
\newcommand{\mm}{September}
\newcommand{\dd}{2}
\newcommand{\authors}{J.~M.~Diederik Kruijssen}
\newcommand{\fulltitle}{Flawed in Nature, Perfect through Evolution}
\newcommand{\shorttitle}{Flawed in Nature, Perfect through Evolution}

\newcommand{\doi}{10.70235/allora.0x\volume\ifnum\numexpr\firstpage<10 000\else\ifnum\numexpr\firstpage<100 00\else\ifnum\numexpr\firstpage<1000 0\fi\fi\fi\firstpage}
\fancypagestyle{firstpage}{%
    \fancyhead[L]{\footnotesize Allora Decentralized Intelligence \textbf{\volume}, \firstpage--\lastpage; \yyyy\ \mm\ \dd}
    \fancyhead[R]{\footnotesize doi:\href{https://doi.org/\doi}{\textcolor{blue}{\doi}}}
    \fancyfoot{}
    
}

\AtBeginDocument{\thispagestyle{firstpage}}

\newtheorem{theorem}{Theorem}

\begin{document}

\title{\fulltitle}
\author{\authors}
\date{\monthyeardate{\today}}


\vskip30mm
\begin{center}
\begin{minipage}{170mm}
\begin{center}
\vskip5mm
{\fontsize{15pt}{15pt}\textbf{Flawed in Nature, Perfect through Evolution}}
\vskip5mm
J.~M.~Diederik Kruijssen$^{\orcidlink{0000-0002-8804-0212}}$
\vskip1mm
\textit{Allora Foundation}
\end{center}
\end{minipage}
\end{center}
\vspace{3mm}

\begin{abstract}
\noindent
The performance of artificial intelligence (AI) and machine learning (ML) models degrades when the problem they were trained on drifts. This is a near-universal feature of real-world problems, which often change unpredictably. Biological evolution has achieved intelligence by overcoming this obstacle through natural selection acting on heritable variation. AI/ML techniques have long incorporated forms of natural selection, but it has been challenging to maintain model diversity as optimization naturally drives convergence. Here we show that a swarm of AI/ML models subjected to deliberate mutations of their model coefficients away from optimality can reliably and sustainably improve performance in changing environments by acting as a statistical hedge against non-stationarity. We call this mechanism `Flawed in Nature, Perfect through Evolution', reflecting that the collective performance gain goes at the expense of individual performance. We prove via four theorems that the resulting regret reduction is guaranteed under general conditions, establishing the Flawed-in-Nature mechanism as a generalizable design principle for AI/ML systems. We validate these results on synthetic linear regression tasks, demonstrating that the mutated swarm delivers the best model in $\sim80\%$ of environment changes and that inference synthesis successfully translates this individual advantage into a collective one. The mechanism proves to be most effective when the mutation drift rate matches the drift rate of the environment. We outline a simple, adaptive controller that enables practical applications by tuning the mutation drift rate to match the unknown drift rate of the environment. The close analogy of the Flawed-in-Nature mechanism to biological evolution suggests it may have been a critical missing ingredient for the organic discovery of AI forms that more closely mimic biological intelligence.
\end{abstract}
\vspace{3mm}


\section{Introduction} \label{sec:intro}
The fields of artificial intelligence (AI) and machine learning (ML) attempt to generate forms of intelligence that did not arise through natural processes. Evaluating achievements in these fields often requires formalizing the concept of intelligence in a way that is not necessarily tied to natural intelligence. Numerous definitions of intelligence have been proposed, perhaps most famously by \citet{legg07}, who define intelligence as the ability to achieve goals in a wide range of environments. While Legg-Hutter intelligence is not directly computable due to its unbounded summation over all possible environments, it provides a useful conceptual starting point for broader reflection on the nature of (machine) intelligence. One of the key requirements for achieving goals in real-world settings is the ability to generalize to unfamiliar circumstances. This requires concrete abilities, such as originality, creativity, foresight, and many others. Each of these qualities contributes to an agent's intelligence by letting it adapt to its (possibly unknown) environment.

The recent success of transformer-based models in natural language processing and computer vision \citep[e.g.][]{vaswani2017attention} has reignited debates about the nature of intelligence, heralding the achievements of GPT, Claude, Gemini, and other frontier models as indications of true intelligence or even consciousness. The echoes of \citet{searle80}'s Chinese Room persist, as impressive outputs need not imply genuine adaptation to novel circumstances. Large language models (LLMs) are able to generate text that is often indistinguishable from human-generated content, but remain fundamentally constrained by their training data. Originality, creativity, and foresight under changing environments thus pose a fundamental challenge to monolithic AI models. In this paper, we prove that this challenge is formally irreducible for any single model, but can be overcome by a population of deliberately diversified models.

Arguably, the emergence of intelligence has been solved in nature. Insofar as our species can be deemed intelligent, our standard is the only practical point of reference we have for defining intelligence. It is therefore natural to ask how nature generated intelligence and to try and replicate these processes in artificial systems.

Evolution \citep{darwin59} stands at the basis of intelligence. It relies on the two fundamental mechanisms of natural selection and heritable variation. Natural selection provides the feedback loop needed to evaluate performance (originally `fitness') and eliminate individuals that do not perform well. In AI, forms of natural selection are routinely achieved through ensemble methods or decentralized learning \citep[e.g.][]{breiman96,dietterich00,rao21,kruijssen24}. However, heritable variation is not naturally present in modern AI systems.

The challenge with heritable variation is that many AI or ML models are single, monolithic entities that are intrinsically incapable of generating variation. Even when considering a population (or swarm) of models, their collective goal of optimizing a loss function given a set of training data implies a convergent pressure that limits variation over time, rather than promoting it. This convergence has certain parallels with the evolution of life, where natural selection considered in isolation also increases the homogeneity of a population over time. Nature solves this problem by enhancing variation through two mechanisms: genetic recombination and mutation. The former is achieved through sexual reproduction, which combines the genetic material of two individuals. It is less suitable for artificial systems, because models typically do not generate offspring through parameter combination (although interesting exceptions exist, see e.g.\ \citealt{stanley02,wortsman22}). However, the process of mutation is a natural way to introduce variation into a population of models.

For single, monolithic models, mutation is not constructive. After all, it perturbs the model away from optimality achieved through training. This means that any form of mutation is inherently detrimental to the model's performance. However, when considering a population of models, the goal changes from the optimality of a single entity to achieving optimality across the swarm. Any single model may be flawed, but the swarm's collective fitness is boosted by mutation-induced diversity, potentially to the point that a swarm of mutated models outperforms a swarm of original, optimized models. This is the central idea presented in this paper, which we name `Flawed in Nature, Perfect through Evolution' (or `Flawed-in-Nature' for short). The parallels with the evolution of life are evident and deliberate.

Analogous ideas have been applied in ensemble methods, which often generate model diversity through random perturbations of hyperparameters \citep[e.g.][]{jaderberg17}, random initialization \citep{lakshminarayanan17}, random feature subsampling \citep{breiman01}, independent fine-tuning runs \citep[e.g.][]{wortsman22}, or inference-time noise \citep[e.g.][]{srivastava14,gal16}. Similarly, evolutionary computation techniques have typically employed mutations to search for optima in a static environment, where population diversity is a transient way of reaching convergence rather than a persistent objective \citep[e.g.][]{stanley02,salimans17}. Continual learning methods \citep[e.g.][]{kirkpatrick17,zenke17} instead address learning in changing environments. They do so by updating a single model, while mitigating the loss of previously learned knowledge (catastrophic forgetting) through weight consolidation. The Flawed-in-Nature mechanism takes a fundamentally different approach. Rather than adapting one model to track a changing environment, it maintains a population of deliberately diversified models and relies on post-hoc selection of the best-performing member.

No previous work has attempted to systematically generate heritable (i.e.\ persistent) variation through deliberate, cumulative parameter mutations, while providing theoretical guarantees of performance improvement under non-stationarity. It is the goal of this paper to demonstrate that improvement. The benefits of mutations are expected to apply specifically in environments that change unpredictably (often referred to as `concept drift' in the literature, see e.g.\ \citealt{gama14}). This is a fundamental difference between the Flawed-in-Nature mechanism and previous work on ensemble methods, evolutionary computation, and continual learning. As in nature, we expect deviations from optimality to be advantageous mostly as a statistical hedge against non-stationarity.

The structure of this paper is as follows. In \S\ref{sec:proof}, we define the Flawed-in-Nature mechanism and demonstrate its viability by proving a set of four theorems that systematically quantify the expected performance of a single model and a swarm of models in static and changing environments. In \S\ref{sec:inference}, we describe how the Flawed-in-Nature mechanism can be realized in practice through a simple inference-synthesis layer that aggregates the outputs of a swarm of mutated models. In \S\ref{sec:experiments}, we validate the theoretical predictions through numerical experiments on a suite of synthetic linear regression problems. In \S\ref{sec:discussion}, we discuss the limitations of the experiments, the outline of a self-learning controller that enables the practical application of the Flawed-in-Nature mechanism to realistic environments with unknown drift, and the impact of mutating AI/ML models on the future of swarm intelligence. We conclude in \S\ref{sec:conclusion} with a summary of our main findings and their broader implications.

\section{Optimization of Single Models and Swarms Differs Fundamentally} \label{sec:proof}
In this section, we show that model diversity is unnecessary when applied to a static environment. The added value of the Flawed-in-Nature mechanism pertains to the practically relevant case in which the environment changes unpredictably, turning stochastic parameter mutations into a lasting strategic advantage. We prove its viability by proving a set of four theorems that revolve around the following central hypotheses:
\begin{enumerate}[itemsep=-3pt]
  \item A single model in a static environment achieves empirical risk minimization (ERM) and converges to the risk minimizer.
  \item A single model in a changing environment incurs linear regret and thereby reaches obsolescence.
  \item A single model in a changing environment is still information-theoretically optimal given its access to past observations only, and its excess loss is therefore irreducible.
  \item A swarm of mutating models in a changing environment achieves strictly smaller regret than any single model, thereby breaking the linear regret bound that constrains individual models.
\end{enumerate}
These hypotheses imply that the Flawed-in-Nature mechanism introduces a form of model population diversity that becomes advantageous once the environment experiences drift. The key elements of the mechanism are illustrated in \autoref{fig:flawed_in_nature}.
\begin{figure}[!t]
  \centering
  \includegraphics[width=\textwidth]{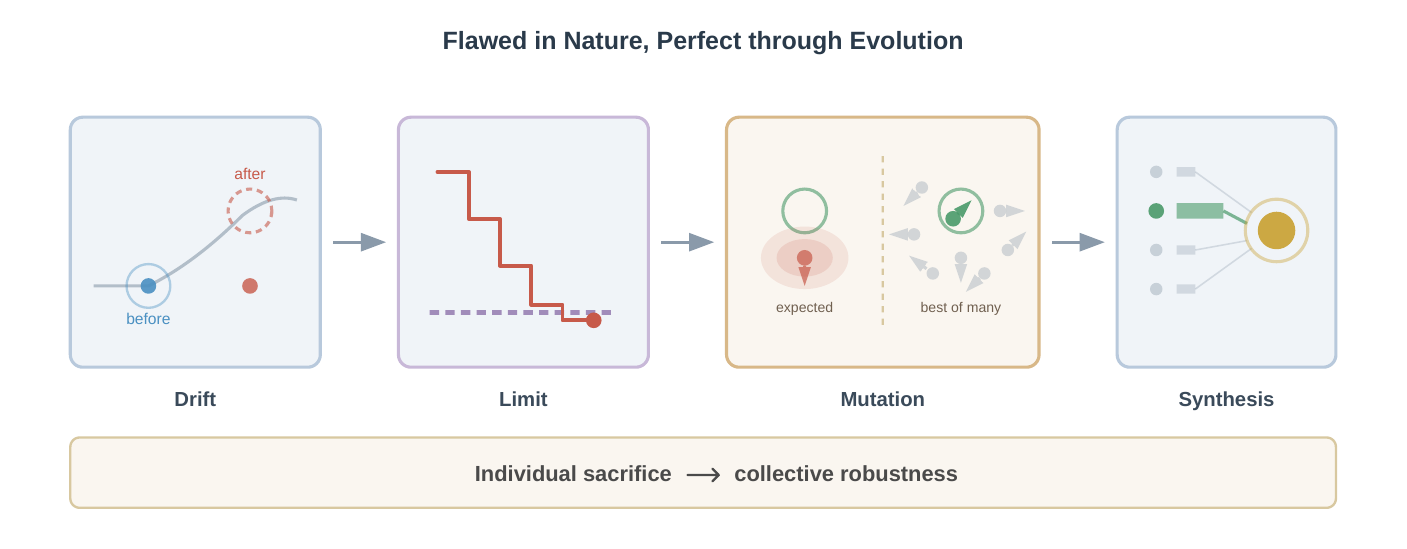}
  \caption{Schematic illustration of the Flawed-in-Nature mechanism. When conditions change unpredictably, a model trained on past data becomes obsolete as the environment has drifted, and the model has been unable to adapt. A single model therefore experiences an error with each unexpected shift that retraining cannot prevent, resulting in a fundamental limit to its accuracy. Introducing random variation (`mutations') into model parameters makes the expected accuracy of any individual model worse, yet offers a critical advantage to the swarm: its ensuing model diversity implies that some mutated model will be well-positioned for the new conditions that emerge. Weighting model outputs by recent accuracy ensures the best-positioned model variant dominates the aggregate prediction. This results in an elevated degree of accuracy that no single model can sustain, and is achieved through a form of random model differentiation that harms individuals but benefits the collective.}
  \label{fig:flawed_in_nature}
\end{figure}

\subsection{Empirical-risk Minimisation of a Single Model in a Static Environment}
\label{sec:single_static}
The central hypothesis of the Flawed-in-Nature mechanism is that persistent mutations of model parameters generate a form of model population diversity that becomes advantageous once the environment experiences drift. We quantify that advantage by first establishing a precise benchmark describing the best attainable performance in the absence of such drift. This is accomplished by showing that, under standard regularity conditions, ERM converges to the risk minimizer within a fixed hypothesis class whenever the joint distribution of features and targets is time-invariant.

\subsubsection{Notation and Basic Definitions}
Let $(X,Y)$ be random variables taking values in $\mathbb R^{d}\times\mathbb R$ and distributed according to a fixed probability measure~$\mathcal D$. Denote by $(x_i,y_i)_{i\ge1}$ an independent and identically distributed (i.i.d.) sample from~$\mathcal D$. Unless stated otherwise, the underlying probability space is suppressed from the notation, i.e.\ probabilities and expectations refer to~$\mathcal D$. Formally, we have:
\be \label{eq:data_model}
  (X,Y) \sim \mathcal D , \quad Y = f(X)+\epsilon , \quad \mathbb E[\epsilon\mid X] = 0 ,
\ee
where $f$ is the true regression function and $\epsilon$ represents the irreducible error. Under squared    
loss, the irreducible Bayes risk equals $\mathrm{Var}(\epsilon)$.

A learning algorithm represents a function drawn from a hypothesis class, which represents a family of candidate predictors $\mathcal H=\{h_\theta:\mathbb R^{d}\to\mathbb R\mid\theta\in\Theta\}$, with $\Theta$ the parameter space spanning all possible model configurations. No specific parametric form is imposed. Throughout, we make the standard assumption that $\mathcal H$ is measurable and separable.

We then define a measurable loss function $\ell:\mathbb R\times\mathbb R\to\mathbb R_{\ge0}$ with $\sup_{\hat{y},y}\ell(\hat{y},y)\le B<\infty$ with $B\ge 0$.\footnote{Technically, a finite second moment suffices \citep[e.g.][]{pollard84}, but we assume bounded losses for simplicity.} This function quantifies how wrong a prediction $\hat{y}$ is when the truth is $y$. The risk (i.e.\ the expected loss on new data) of any predictor $h\in\mathcal H$ is then defined as:
\be \label{eq:true_risk}
  R(h) = \mathbb E_{(X,Y)\sim\mathcal D}\!\left[\ell\left(h(X),Y\right)\right] .
\ee
Because $\mathcal D$ is unknown, we approximate \autoref{eq:true_risk} by the empirical loss on the first $n$ observations:
\be \label{eq:empirical_risk}
  \widehat R_n(h) = \frac1n\sum_{i=1}^{n}\ell\left(h(x_i),y_i\right) .
\ee
Given these observations, the hypothesis or model that minimizes the empirical risk is defined as:
\be \label{eq:erm_minimizer}
  h_n = \arg\min_{h\in\mathcal H}\widehat R_n(h) .
\ee
This is the best model according to the training data.

The proof that ERM converges to the risk minimizer relies on two statistical learning-theoretic concepts that quantify the complexity of a model class.
\begin{enumerate}
  \item
  The Vapnik-Chervonenkis (VC) dimension measures the complexity of a model class by counting how many different ways the models can classify a set of points. Specifically, it is the largest number of points that can be classified in all possible ways by models in the class (e.g.\ linear separators in $\mathbb R^{2}$ separate at most three points). For real-valued functions (like regression models), we consider whether the function exceeds a threshold, effectively converting it to a binary classification problem.

  Formally, the VC-dimension of a class $\mathcal H$ of $\{0,1\}$-valued functions is the maximal integer~$m$ such that some set of $m$ points in $\mathbb R^{d}$ is assigned all $2^{m}$ binary labelings by members of $\mathcal H$. For real-valued classes one considers the induced family $\{\mathbf 1_{h(x)\ge t}:h\in\mathcal H,\,t\in\mathbb R\}$ to first convert the real-valued functions to binary labelings.
  
  \item
  A class $\mathcal H$ is Glivenko-Cantelli (GC) if the empirical risk converges uniformly to the true risk as the sample size grows. This means that for any model in the class, the difference between its performance on the training data and its expected performance on new data becomes arbitrarily small as we collect more data.
  
  Formally, this requires that the maximum difference between empirical and true risk across all models in the class converges to zero:
  \be
    \sup_{h\in\mathcal H}\left|\widehat R_n(h)-R(h)\right|
    \xrightarrow[n\to\infty]{\mathrm{a.s.}} 0 ,
  \ee
  where `a.s.' denotes almost sure convergence.
  
  This property ensures that training performance reliably predicts future performance. A finite VC-dimension is a sufficient condition for the GC property if the losses are bounded \citep[e.g.][]{vapnik98}. However, the GC property and the convergence results that follow only hold when the data-generating distribution $\mathcal D$ is static (time-invariant). When the distribution evolves over time, these guarantees no longer apply, which is precisely why the Flawed-in-Nature approach becomes advantageous in changing environments (see \S\ref{sec:swarm_change}).
\end{enumerate}

\subsubsection{Consistency of ERM}
We now prove that ERM converges to the risk minimizer within a fixed hypothesis class whenever the joint distribution of features and targets is time-invariant.
\begin{theorem}[ERM consistency under stationarity] \label{thm:erm_consistency}
Let the loss $\ell$ be bounded and the hypothesis class $\mathcal H$ possess a finite VC-dimension (i.e.\ it is GC), as appropriate for a static environment with a time-invariant data-generating distribution $\mathcal D$. Define $R(\,\cdot\,)$ and $h_n$ as in \autoref{eq:true_risk} and \autoref{eq:erm_minimizer}, and set $h^\star:=\arg\min_{h\in\mathcal H}R(h)$. Then we have:
\be \label{eq:erm_consistency}
  R(h_n) \xrightarrow[n\to\infty]{\mathrm{a.s.}} R(h^\star) .
\ee
In other words, as the sample grows, the empirically-trained model reaches the lowest possible expected loss attainable inside $\mathcal H$.
\end{theorem}

\begin{proof}
  Since $\mathcal H$ is GC,
  $\Delta_n:=\sup_{h\in\mathcal H}\lvert\widehat R_n(h)-R(h)\rvert\xrightarrow{\mathrm{a.s.}}0$. $\Delta_n$ is a uniform deviation bound of every model in the class.
  For any fixed $n$ we write:
  \be
    R(h_n)-R(h^\star)
      =\underbrace{R(h_n)-\widehat R_n(h_n)}_{\le \Delta_n}
       +\underbrace{\widehat R_n(h_n)-\widehat R_n(h^\star)}_{\le 0}
       +\underbrace{\widehat R_n(h^\star)-R(h^\star)}_{\le \Delta_n} .
  \ee
  The first and third terms represent the estimation error of $h_n$ and $h^\star$, respectively. They are bounded in magnitude by $\Delta_n$ (i.e.\ $|R(h_n)-\widehat R_n(h_n)| \le \Delta_n$ and $|\widehat R_n(h^\star)-R(h^\star)| \le \Delta_n$) and converge to zero because $\mathcal H$ is GC (uniform convergence). The second term is non-positive because $h_n$ minimizes the empirical risk by construction. Combining all three terms, we obtain $0 \le R(h_n) - R(h^\star) \le 2\Delta_n$ for every $n$. Because $\Delta_n \to 0$ almost surely, it follows that $R(h_n) \to R(h^\star)$. Since $h^\star$ is the best model in $\mathcal H$ by definition, no other hypothesis in the same architecture can achieve strictly smaller long-run loss than the empirical minimizer $h_n$.
\end{proof}

\autoref{thm:erm_consistency} shows that, under stationarity (i.e.\ a fixed distribution $\mathcal D$) and standard learning-theoretic assumptions, the empirical minimizer attains the minimal achievable risk within~$\mathcal H$. Under the hypotheses of \autoref{thm:erm_consistency}, any sequence of random perturbations applied to \(h_n\) after training yields, in expectation, a risk no smaller than \(R(h_n)\) as \(n\to\infty\). In other words, random perturbations of its parameters, or equivalently mutations in the Flawed-in-Nature terminology, increase the expected loss. Consequently, any sequence of random perturbations applied after ERM yields, in expectation, risk no smaller than $R(h_n)$. As such, \autoref{thm:erm_consistency} delineates the regime in which the Flawed-in-Nature mechanism cannot improve performance, i.e.\ when the data-generating mechanism is time-invariant. However, real-world problems are rarely stationary. \S\ref{sec:single_change}-\ref{sec:swarm_change} consider situations in which the distribution~$\mathcal D$ evolves with time. In that setting, the conclusions of \autoref{thm:erm_consistency} no longer hold and diversity may become advantageous.

\subsection{Single Model in a Changing Environment} \label{sec:single_change}
In \S\ref{sec:single_static}, we established that under a time-invariant data-generating distribution~$\mathcal D$, the ERM $h_n$ converges almost surely to the risk minimizer inside the hypothesis class~$\mathcal H$ (\autoref{thm:erm_consistency}). However, real-world data streams rarely satisfy the assumption of stationarity: sensor calibrations may shift abruptly, market regimes can flip within minutes, and user behavior models may become obsolete after an external news event. Here we formalize such environmental drift, starting with a minimal yet sufficiently general model of change. We then analyze the performance limits of any single predictor that is updated only through empirical optimization on past data.

\subsubsection{Definition of Changing Environment} \label{sec:drift}
We consider discrete time, indexed by $t=1,2,\dots$.  
At each time step a new observation $(X_t,Y_t)\in\mathbb R^{d}\times\mathbb R$
is generated from a distribution $\mathcal D_t$ that may vary with~$t$. The drawn observations are conditionally independent given the sequence $\{\mathcal D_t\}_{t\ge1}$.

For this set of observations, we define a piecewise-stationary sequence of data-generating distributions $\{\mathcal D_t\}_{t\ge1}$ with the following properties:
\begin{enumerate}
  \item \emph{Segment structure.} There exists an increasing sequence of change-points $0=\tau_0<\tau_1<\tau_2<\dots$ such that $\mathcal D_t$ is constant on every half-open interval $[\tau_k,\tau_{k+1})$. The inter-arrival times $\{\tau_{k+1}-\tau_k\}_{k\ge 0}$ are i.i.d.\ with finite mean $1/\lambda>0$ and finite second moment $\mathbb E[(\tau_{k+1}-\tau_k)^2]<\infty$. Each jump can be regarded as an impulse applied to the underlying mapping between~$X$ and~$Y$, after which the environment is stationary until the next impulse. The distribution that applies during the $k$-th segment is referred to as $\mathcal D^{(k)}=\mathcal D(\beta^{(k)})$, where $\mathcal D(\cdot)$ denotes a measurable map from the parameter space to the corresponding data-generating distribution and $\beta^{(k)}$ is the parameter vector that applies during segment $k$.
  \item \emph{Unpredictability of jumps.} We define an increment $\delta_k=\beta^{(k)}-\beta^{(k-1)}$, which determines how the parameter vector changes between segments as $\beta^{(k)}=\beta^{(k-1)}+\delta_k$. The increment $\delta_k$ is independent of all observations made before $\tau_k$. As a result, neither the direction nor the magnitude of each jump can be inferred from pre-change data, even though the new distribution $\mathcal D^{(k)}=\mathcal D(\beta^{(k)})$ still depends on the previous one through the term $\beta^{(k-1)}$.
  \item \emph{Stationarity within segments.} Conditional on $\mathcal D^{(k)}$, the observations inside a segment are i.i.d.:
  \be
    (X_t,Y_t)\mid \left\{t\in[\tau_k,\tau_{k+1})\right\} \sim \mathcal D^{(k)} ,
  \ee
  with target data defined as
  \be
    Y_t = f^{(k)}(X_t) + \epsilon_t,
    \qquad
    \mathbb E[\epsilon_t\mid X_t] = 0 ,
  \ee
  where $f^{(k)}$ is the segment-specific true regression function and $\epsilon_t$ represents the irreducible error at time $t$. We do not make any structural assumptions on the segment-specific regression function $f^{(k)}$ aside from measurability.
\end{enumerate}
The results of this work only require these properties. No moment bounds on the increments of the data-generating distribution are assumed. Given these properties, any drift that is continuous in time can be approximated as the limit of a sequence of infinitesimal impulses. Therefore, this definition also covers the general case.

The second of the above properties implies that a model fitted before $\tau_k$ cannot exploit its pre-jump observations to predict the distribution in the next segment, because the jump increment is independent of prior data. A particularly simple instantiation (which we will adopt in \S\ref{sec:experiments}) is obtained by choosing the change points as the arrival times of a Poisson process with rate $\lambda>0$, and generating the post-change distribution by adding an independent random increment to the parameter vector that characterizes $\mathcal D^{(k)}$.

With the above definition of the data-generating distribution, we now define the performance metric that will be used to compare fixed and mutating models. At time $t$, the instantaneous (population) risk of a predictor $h:\mathbb R^{d}\to\mathbb R$ is defined analogously to \autoref{eq:true_risk}:
\be
  R_t(h) = \mathbb E_{(X_t,Y_t)\sim\mathcal D_t}\left[\ell\left(h(X_t),Y_t\right)\right] .
\ee

The Bayes-optimal risk at time $t$ is
\be
  \label{eq:inf_risk}
  R_t^\star = \inf_{h\in\mathcal H}R_t(h) ,
\ee
which represents the best possible performance achievable by any model in the hypothesis class at time $t$. Using these two risks, we can express the performance over a time horizon $T$ as the cumulative regret \citep[following the standard framework of][]{cesa06}:
\be
  \label{eq:regret}
  \mathcal{R}_T(h) = \sum_{t=1}^{T}\left[R_t(h)-R_t^\star\right] .
\ee
Because $R_t(h) \ge R_t^\star$ for all $h\in\mathcal H$ by definition of the infimum, the regret is non-negative. A positive regret indicates that the model cannot achieve the best possible performance at each time step.

In the remainder of \S\ref{sec:single_change}, we examine how models defined by a fixed parameter vector or those that are updated only by re-optimizing on past data incur positive regret under this form of drift. In \S\ref{sec:swarm_change}, we then demonstrate that such regret can be reduced to zero by introducing population diversity.

\subsubsection{A Single Optimized Model Inevitably Reaches Obsolescence} \label{sec:obsolescence}
\autoref{thm:erm_consistency} guarantees optimal long-run performance only when the data-generating mechanism is fixed. We now show that under the drift model of \S\ref{sec:drift} (specifically Properties~1--2, i.e.\ segment structure and unpredictable increments), any predictor that remains tied to a single parameter vector (no matter how often it is re-estimated from past data) accumulates regret that grows at least linearly in time \citep[consistent with known lower bounds in non-stationary online learning, e.g.][]{besbes15}. In short, the model becomes obsolete the moment an unpredictable jump occurs.

Let $h_{1:T} = (h_1, \ldots, h_T)$ be a sequence of predictors produced by any learning procedure that only uses information available up to time~$t-1$ when selecting $h_t$. This general case includes classical ``re-train-from-scratch'' and ``incremental-update'' strategies, as well as the static model $h_t\equiv h$. As before, we denote by $R_t(h)$ and $R_t^\star$ the instantaneous risk and Bayes risk defined in \autoref{eq:true_risk} and \autoref{eq:inf_risk}. Analogously to \autoref{eq:regret}, the cumulative (expected) regret over~$T$ rounds is then
\be
  \mathbb E\!\left[\mathcal{R}_T(h_{1:T})\right] = \sum_{t=1}^{T}\mathbb E\!\left[R_t(h_t)-R_t^\star\right] .
\ee

\begin{theorem}[Linear regret under independent impulse drift] \label{thm:linear_regret}
Assume the drift model of \S\ref{sec:drift} with jump times $\{\tau_k\}$ and increments $\{\delta_k\}$ that satisfy Property~2 (independence of $\delta_k$ from all observations made before $\tau_k$). Let $\ell$ be a bounded loss and let $\rho=\mathbb E\!\left[R_{\tau_k}(h_{\tau_k}) - R_{\tau_k}^\star\right] > 0$ (independent of $k$) denote the expected instantaneous performance gap immediately after a jump.\footnote{The bound $\rho>0$ is ensured by Property~2: conditional on the pre-change data, the increment $\delta_k$ is non-degenerate, so the post-change Bayes predictor changes in expectation, whereas $h_{\tau_k}$ has no information to anticipate the move.} If the expected segment length $\mathbb E[\tau_{k+1}\!-\!\tau_k] = 1/\lambda$ is finite (e.g.\ Poisson jumps with rate $\lambda>0$), then there exists\footnote{If the learner re-fits sufficiently fast within each stationary segment, the per-jump adaptation tail has finite expectation, so $\rho'=\rho+\bar a$ with $\rho>0$ the expected immediate post-jump excess risk and $\bar a\ge 0$ the expected per-jump adaptation contribution. In all cases where the expected number of non-degenerate jumps grows linearly with $T$, the expected regret grows linearly as well, i.e.\ $\mathbb E[\mathcal{R}_T]=\Theta(T)$.} a constant $\rho'\ge\rho>0$ such that for any learning rule based on past data alone
\be \label{eq:linear_regret}
  \mathbb E\!\left[\mathcal{R}_T(h_{1:T})\right] \ge \rho'\,\lambda\,T - O(1) ,
  \quad \text{as } T\to\infty .
\ee
In other words, the expected regret grows linearly with~$T$ at a rate at least $\rho'\lambda$, with a residual term that is $O(1)$ (e.g.\ due to any finite start-up transients or a partial last segment).
\end{theorem}

\begin{proof}
We partition the time axis by the change-points $0=\tau_0<\tau_1<\dots$. Immediately after each jump we incur an instantaneous excess loss 
\be
  g_k = R_{\tau_k}(h_{\tau_k}) - R_{\tau_k}^\star,  \quad  \mathbb E[g_k] = \rho .
\ee
Inside a segment $[\tau_k,\tau_{k+1})$, the predictor $h_t$ may improve as it incorporates fresh data. We express the regret accumulated within that segment as
\be
  \Delta_k = \sum_{t=\tau_k+1}^{\tau_{k+1}-1}\left(R_t(h_t)-R_t^\star\right) \ge 0 .
\ee
We can now express the total regret as
\be
  \mathcal R_T(h_{1:T}) = \sum_{k:\,\tau_k\le T} g_k \;+\; \sum_{k:\,\tau_{k+1}\le T} \Delta_k \;+\; O(1) ,
\ee
where the $O(1)$ term comes from the truncated last segment. Because jumps occur with a rate~$\lambda$, the expected number of jumps in the first $T$ rounds is $\mathbb E[N_T]=\lambda T + O(1)$. Linearity of expectation then yields
\be
  \mathbb E\!\left[\mathcal{R}_T(h_{1:T})\right]
  =
  \sum_{k:\,\tau_k\le T}\mathbb E[g_k] + \sum_{k:\,\tau_{k+1}\le T}\mathbb E[\Delta_k]
  =
  \rho \mathbb E[N_T] + \bar a \mathbb E[N_T] + O(1)
  =
  (\rho + \bar a) \lambda T + O(1) ,
\ee
where the summations run over completed segments and $\bar a = \mathbb E[\Delta_k] \ge 0$ is the expected per-jump adaptation contribution. Setting $\rho' = \rho + \bar a$ with $\rho'>0$ yields a lower bound of the form $\mathbb E\!\left[\mathcal{R}_T(h_{1:T})\right] \ge \rho'\lambda T - O(1)$ (the inequality follows because $O(1)$ can be absorbed with a sign change) as in \autoref{eq:linear_regret}.
\end{proof}

\autoref{thm:linear_regret} formalizes the intuitive idea that a single model cannot anticipate an independent jump in the environment, no matter how often it is re-trained on past data. Every such jump injects a positive loss that accumulates over time, yielding a regret that grows linearly with time.\footnote{A similar lower bound appears in the concept-drift literature (e.g.\ \citealt{duchi19} for square loss), but we phrase it here in a form tailored to the impulse drift model.} The linear lower bound shows that the performance guarantee of \autoref{thm:erm_consistency} fails as soon as even rare, unpredictable drifts are allowed. Before demonstrating that a swarm of mutating models can break the linear lower bound, we first clarify in \S\ref{sec:optimality} that a single model can nonetheless remain individually optimal given its information. This implies that the inescapable regret arises from the informational handicap, not from sub-optimal optimization.

\subsubsection{A Single Optimized Model Still Minimizes Risk} \label{sec:optimality}
The linear lower bound of \autoref{thm:linear_regret} does not arise from any avoidable optimization deficit. Instead, it is a purely informational phenomenon that reflects a fundamental limitation of any learning rule that only has access to data up to time $t-1$. Once a jump occurs, the post-change distribution $\mathcal D^{(k)}$ (and thus the new Bayes predictor) differs from the pre-change one through an increment that is independent of all pre-jump observations (Property~2 in \S\ref{sec:drift}). A learning rule that only has access to data up to time $t-1$ cannot condition on future increments, and therefore cannot consistently pre-empt the shift. We now formalize this statement by showing that, at every time $t$, an empirically re-optimized single model attains the minimal conditionally achievable risk given the information available at $t-1$. The residual excess loss immediately after a jump is therefore irreducible.

Let $\mathcal F_t$ denote the collection of all observations $\{(X_s,Y_s)\}_{s\le t}$ available up to time $t$. A predictor $h_t$ deployed at time $t$ can only use past observations (those in $\mathcal F_{t-1}$). For any such predictor, the conditional instantaneous risk is defined as
\be
  \widetilde R_t(h_t) = \mathbb E\!\left[\, \ell\left(h_t(X_t),Y_t\right)\,\middle|\, \mathcal F_{t-1}\right] ,
\ee
where the tilde denotes that the risk is conditional on the information available up to time $t-1$. The corresponding (conditional) Bayes-optimal risk at time $t$ is defined analogously to \autoref{eq:inf_risk} as
\be
  \widetilde R_t^\star = \inf_{h\in\mathcal H} \widetilde R_t(h) .
\ee
We assume $\mathcal H$ contains the segment-wise Bayes predictor $f^{(k)}$ for each segment. Then it follows that
\be
  \widetilde R_t^\star
  = \inf_{h\in\mathcal H} \mathbb E\!\left[\, \ell\left(h(X_t),Y_t\right)\,\middle|\, \mathcal F_{t-1}\right] .
\ee

\begin{theorem}[Informational optimality of a single model] \label{thm:info_optimality}
Consider the drift model of \S\ref{sec:drift} (Properties~1--3). Let $h_t$ be obtained at each time $t$ by empirical-risk minimisation over the data $\mathcal F_{t-1}=(X_s,Y_s)_{s\le t-1}$ (with any tie-breaking rule), i.e.\ $h_t$ is $\mathcal F_{t-1}$-measurable. Then for every $t$,
\be
  \mathbb E[\widetilde R_t(h_t) - \widetilde R_t^\star] \leq \epsilon(n_k) \quad \text{with} \quad \epsilon(n_k) \rightarrow 0 \quad \text{as} \quad n_k \rightarrow \infty ,
\ee
where $n_k$ denotes the number of observations accumulated within segment $k$ and $\epsilon(n_k)\geq 0$ is deterministic. The limit represents the case where sufficient data has been observed within the stationary segment. Then $h_t$ attains the minimal conditional risk given $\mathcal F_{t-1}$ in the limit as the number of observations within the current segment grows. Nevertheless, on jump times $\tau_k$ we have the unconditional excess loss
\be
  \mathbb E\!\left[ R_{\tau_k}(h_{\tau_k}) - R^\star_{\tau_k} \right] = \rho > 0 ,
\ee
as in \autoref{thm:linear_regret}. Because the increment $\delta_k$ is independent of $\mathcal F_{\tau_k-1}$, the excess conditional (hence unconditional) loss at jump times is irreducible under the information constraint.
\end{theorem}

\begin{proof}
At any given time $t$, conditional on the current segment (i.e.\ on the event $\{t\in[\tau_k,\tau_{k+1})\}$ for some $k$) and given $\mathcal F_{t-1}$, the distribution of $(X_t,Y_t)$ is $\mathcal D^{(k)}$ by Property~1 from \S\ref{sec:drift}. The conditional Bayes-optimal predictor for this segment is the (segment-specific) function
\be
h^{\star,(k)} = \arg\min_{h \in \mathcal H} \mathbb E\!\left[\,\ell(h(X_t), Y_t) \,\middle|\, \mathcal F_{t-1}, t \in [\tau_k, \tau_{k+1})\,\right] ,
\ee
where $h^{\star,(k)}$ is $\mathcal F_{t-1}$-measurable and depends only on $\mathcal D^{(k)}$ (not on future data). Since within each segment the distribution is stationary (Property~3 from \S\ref{sec:drift}), this reduces to the standard ERM problem from \S\ref{sec:single_static}. Because $h_t$ is obtained by ERM on past data from the current segment, and assuming sufficient data has been observed within the segment, \autoref{thm:erm_consistency} implies that $h_t$ approaches $h^{\star,(k)}$ in risk. Therefore, we obtain
\be
  \widetilde R_t(h_t) - \widetilde R_t^\star
  = \mathbb E\!\left[\, \ell\left(h_t(X_t),Y_t\right) - \ell\left(h^{\star,(k)}(X_t),Y_t\right)\,\middle|\, \mathcal F_{t-1}\right] \xrightarrow[]{\text{a.s.}} 0 ,
\ee
as the number of observations accumulated inside the segment grows. Taking expectations and letting $n_k\to\infty$ yields $\lim_{n_k\to\infty}\mathbb E[\widetilde R_t(h_t)-\widetilde R_t^\star]=0$.

At a jump time $\tau_k$, the distribution changes from $\mathcal D^{(k-1)}$ to $\mathcal D^{(k)}$ through an increment $\delta_k$ that is independent of $\mathcal F_{\tau_k-1}$ (Property~2 from \S\ref{sec:drift}). Because $\delta_k$ is independent of all past observations, a predictor $h_{\tau_k}$ has no information about the direction or magnitude of this change, as the predictor itself is $\mathcal F_{\tau_k-1}$-measurable (i.e.\ is based only on past data). Therefore, $h_{\tau_k}$ cannot systematically coincide with the Bayes-optimal predictor for $\mathcal D^{(k)}$ immediately after the jump unless the change is completely predictable from past data (contradicting Property~2). This yields a positive expected performance gap $\rho>0$, as defined in \autoref{thm:linear_regret}. The gap is irreducible because any $\mathcal F_{\tau_k-1}$-measurable predictor faces the same informational limitation.
\end{proof}

Using \autoref{thm:info_optimality}, we decompose the cumulative regret into contributions that are (1) conditionally unavoidable given the information $\mathcal F_{t-1}$ and (2) the additional loss incurred exactly at the unpredictable jump times. We define instantaneous conditional regret (i.e.\ the instantaneous conditional excess risk at time $t$ that is due to information available only up to time $t-1$) as
\be
  \widetilde{\mathcal R}_t = \mathbb E\!\left[\, \widetilde R_t(h_t) - \widetilde R_t^\star\,\right] ,
\ee
where the tilde indicates that the variable is conditional on $\mathcal F_{t-1}$ (see \S\ref{sec:optimality}). By the law of total expectation, i.e.\
\be
  \mathbb E[\mathcal R_T(h_{1:T})] = \sum_{t=1}^T \mathbb E[\widetilde{\mathcal R}_t] ,
\ee
\autoref{thm:info_optimality} then implies that $\mathbb E[\widetilde{\mathcal R}_t]\rightarrow 0$ as $n_k\to\infty$ for all non-jump times, while at each jump time $\tau_k$ we have $\mathbb E[\widetilde{\mathcal R}_{\tau_k}] \ge \rho$. Therefore, the linear growth of $\mathbb E[\mathcal R_T(h_{1:T})]$ is driven asymptotically by the unavoidable impulses at unpredictable jump times; the within-segment adaptation period contributes only a subdominant additional linear term (through $\rho'=\rho+\bar a$ in \S\ref{sec:obsolescence}), which does not affect either the order or the leading coefficient of the expectation. Between jumps, the single model is (asymptotically) conditionally Bayes-optimal.

\autoref{thm:info_optimality} shows that the single-model limitation under drift is not due to suboptimal learning: at any given moment, the learner extracts all available information optimally. The fundamental limitation is informational, as no past-data-only predictor can anticipate unpredictable changes. The only way to reduce regret is to expand the information set indirectly and thereby better anticipate future changes. In \S\ref{sec:swarm_change}, we show that the population diversity of a swarm of mutating models achieves this by lowering the expected post-jump gap~$\rho$, as the best-performing mutated model is likely to be closer to the new Bayes-optimal predictor than an unmutated model. This effect reduces the linear regret below the single-model bound.

\subsection{Swarm of Models in a Changing Environment} \label{sec:swarm_change}
The lower bound in \autoref{thm:linear_regret} implies that any single parameter vector (no matter how often re-estimated) incurs linear regret. We now show that a model population equipped with random mutations can provably break this limit.

\subsubsection{Definition of Model Mutation} \label{sec:mutation_def}
We now formalize how a swarm of models can continually grow its population diversity through stochastic parameter mutations. The mutation process is independent of $\mathcal F_{t-1}$ and maintains population diversity by perturbing each model's parameters over time. Let $\Theta\subseteq\mathbb R^p$ be the parameter space of the hypothesis class $\mathcal H=\{h_\theta:\theta\in\Theta\}$. A swarm consists of $N_{\mathrm{m}}\in\mathbb N$ models, indexed by $i\in\{1,\dots,N_{\mathrm{m}}\}$. For a model $i$, a parameter vector at time $t$ is denoted by $\theta_t^{(i)}\in\Theta$ and describes the set of coefficients used to construct model $i$. While we adopt discrete time $t=1,2,\dots$ for consistency with §\ref{sec:drift}, continuous-time versions reduce to this discretized form in the infinitesimal limit.

We describe coefficient-level mutations by fixing a target ``any-mutation'' rate $\mu>0$, i.e.\ the rate at which a model undergoes a mutation of any of its coefficients. In a discrete time step, we define the corresponding per-coefficient mutation probability as
\be
  \label{eq:mutations}
  q_c = \frac{\mu}{p} ,
\ee
which is valid when $\mu \leq p$, and results in an expected number of mutations per model that equals the sum of the expected number of mutations across all of its $p$ individual coefficients.

At each time $t$, each coefficient $j\in\{1,\dots,p\}$ of model $i$ mutates independently with probability $q_c$. Let $Z_{t,j}^{(i)} \in \{0,1\}$ be a binary indicator that equals $1$ if coefficient $j$ mutates, where $\mathbb P(Z_{t,j}^{(i)} = 1) = q_c$. When a coefficient mutates ($Z_{t,j}^{(i)}=1$), it receives an increment $\epsilon_{t,j}^{(i)}\in\mathbb R$, resulting in an instantaneous mutation $m_{t,j}^{(i)}$:
\be
  m_{t,j}^{(i)} = Z_{t,j}^{(i)} \epsilon_{t,j}^{(i)} .
\ee
The mutation increments $\epsilon_{t,j}^{(i)}$ are i.i.d.\ across $t,j,i$. The distribution of $\epsilon_{t,j}^{(i)}$ is assumed to be mean-zero ($\mathbb E[\epsilon_{t,j}^{(i)}] = 0$) and finite-variance ($\mathrm{Var}[\epsilon_{t,j}^{(i)}] < \infty$), which among others is satisfied by a Gaussian distribution $\mathcal N(0,\sigma^2)$. Heavy tails are allowed, as long as the variance is finite.

Mutations accumulate additively in a mutation vector $M_t^{(i)}$:
\be
  M_{t,j}^{(i)} = \sum_{s=1}^t m_{s,j}^{(i)} ,
\ee
such that the parameter vector at time $t$ is given by
\be
  \theta_t^{(i)} = \hat\theta_t^{(i)} + M_t^{(i)} ,
\ee
where $\hat\theta_t^{(i)}$ is the parameter vector of the original, unmutated model trained at time $t$ on the data $\mathcal F_{t-1}$. In other words, each new mutation is added to the existing mutation vector without any resets, and the mutations are additive to and independent of the training update. As such, they represent a form of exploration around a learning trajectory rather than a drift from initial parameters.

We reiterate that the mutation masks and increments are independent of $\mathcal F_{t-1}$ and of the drift increments $\{\delta_k\}$, and mutually independent across $i,t,j$. Mutations that determine $\theta_t^{(i)}$ are sampled at the end of round $t-1$ and are independent of $\mathcal F_{t-1}$. Therefore, $\theta_t^{(i)}$ is fixed before $(X_t,Y_t)$ is realized and is independent of $(X_t,Y_t)$ conditional on $\mathcal F_{t-1}$. Post-jump improvements arise from ex-ante spread, not from extra information.

With the above definitions, we obtain a model-level Bernoulli mutation process that is (approximately) Poisson with rate $\mu$. The probability that at least one coordinate mutates at time $t$ is $1-(1-q_c)^p \approx p q_c = \mu$ for small $q_c$. Hence, $\mu$ emerges from the coordinate-wise clocks rather than being imposed separately.

\subsubsection{A Swarm of Mutated Models Breaks the Linear Regret Bound}
The linear growth of the expected regret for single models (\autoref{thm:linear_regret}) is driven by the instantaneous excess loss injected at unpredictable jump times. We now show that a swarm of mutating models strictly reduces this jump-induced term by ex-ante diversity. Specifically, we demonstrate that multiple independent parameter perturbations yield a strictly smaller expected post-jump gap for the swarm than for the unmutated model.

For the $k$-th stationary segment, we write the risk of a model with parameter vector $\theta$ as
\be
  R^{(k)}(\theta) = \mathbb E\!\left[\,\ell(h_\theta(X),Y)\,\middle|\,t\in[\tau_k,\tau_{k+1})\right] ,
\ee
and we define $\theta^{\star,(k)} \in \arg\min_{\theta\in\Theta} R^{(k)}(\theta)$ as the Bayes-optimal parameter vector for the segment. The segment-wise excess-risk profile around $\theta^{\star,(k)}$ is given by
\be
  G^{(k)}(u) = R^{(k)}\!\left(\theta^{\star,(k)}+u\right) - R^{(k)}\!\left(\theta^{\star,(k)}\right) ,
\ee
where $u \in \mathbb{R}^p$ is a perturbation vector in the parameter space. This function is minimized at $u=0$ and is assumed continuous at $0$, so $G^{(k)}(u)\geq 0$ for all $u$. As in \S\ref{sec:optimality}, ERM within segments implies that just before the jump at $\tau_k$ the unmutated base parameter is $\hat\theta_{\tau_k-1}\approx \theta^{\star,(k-1)}$. At a jump time $\tau_k$, the Bayes parameter vector changes by a random increment vector $\eta_k$ that is independent of pre-jump data by Property~2 of \S\ref{sec:drift}, i.e.
\be
  \eta_k = \theta^{\star,(k)} - \theta^{\star,(k-1)} .
\ee
For a single model, the instantaneous post-jump gap approaches $G^{(k)}(-\eta_k)$ as the pre-jump segment length $\tau_k-\tau_{k-1}\to\infty$ by ERM consistency, and its expectation over the jump increment $\eta_k$ equals $\mathbb E[G^{(k)}(-\eta_k)] = \rho$ from \autoref{thm:linear_regret}.

Under the mutation process of \S\ref{sec:mutation_def}, let $\mathbf{M}_k = \left(M_k^{(1)},\dots,M_k^{(N_{\mathrm m})}\right) \in (\mathbb R^p)^{N_{\mathrm m}}$ denote the collection of cumulative mutation vectors carried by the $N_{\mathrm m}$ models at $\tau_k$. Retaining the unmutated base model within the swarm, the minimum instantaneous excess loss across the swarm at $\tau_k$ is
\be
  L^{\star,(k)}(\mathbf{M}_k,\eta_k) = \min\left\{G^{(k)}(-\eta_k),\;\min_{1\le i\le N_{\mathrm m}} G^{(k)}\!\left(M_k^{(i)} - \eta_k\right)\right\} ,
\ee
where $N_{\mathrm m}$ denotes the number of mutated models (excluding the unmutated base), and the unmutated base is retained explicitly via the $G^{(k)}(-\eta_k)$ term.

\begin{theorem}[Model diversity lowers the expected post-jump gap] \label{thm:swarm_gap}
Assume the drift model of \S\ref{sec:drift} (Properties~1--3) with non-degenerate jumps, i.e.\ $\Pr(\eta_k=0)=0$, and the mutation model of \S\ref{sec:mutation_def}. At each jump $\tau_k$, let $\mathbf{M}_k=(M_k^{(1)},\dots,M_k^{(N_{\mathrm m})})$ be i.i.d.\ across $i$, independent of $\eta_k$ and of $\mathcal F_{\tau_k-1}$, with mean zero and finite variance, and with full support, i.e.\ for any target vector $v\in\mathbb R^p$ and any radius $r>0$, we have $\Pr(\,\|M_k^{(i)}-v\|<r\,)>0$ (with $\|\cdots\|$ representing the Euclidean norm), implying that a mutated model can arrive arbitrarily close to the new Bayes parameter vector. Suppose $G^{(k)}$ is continuous at $0$ for each segment and assume the loss $\ell$ is bounded. Define the expected post-jump gap for a swarm of size $N_{\mathrm m}$ as
\be
  \rho_{N_{\mathrm m}} = \mathbb E\!\left[L^{\star,(k)}(\mathbf{M}_k,\eta_k) \right] ,
  \qquad
  \rho_\mathrm{base} = \mathbb E\!\left[ L^{\star,(k)}(\mathbf{0},\eta_k) \right] = \mathbb E\!\left[ G^{(k)}(-\eta_k) \right] = \rho .
\ee
Expectations are taken over the randomness of $\eta_k$ and $\mathbf{M}_k$, conditional on $\mathcal F_{\tau_k-1}$. Then for any $N_{\mathrm m}\ge 1$,
\be
  0 \le \rho_{N_{\mathrm m}} < \rho_\mathrm{base} ,
\ee
and $\rho_{N_{\mathrm m}}$ is non-increasing in $N_{\mathrm m}$ with
\be
  \lim\limits_{N_{\mathrm m}\to\infty}\rho_{N_{\mathrm m}} = 0 .
\ee
Consequently, the expected jump contribution to $\mathbb E[\mathcal{R}_T(h_{1:T})]$ is $(\lambda T)\rho_{N_{\mathrm m}}+O(1)$, strictly below the single-model $(\lambda T)\rho_\mathrm{base}+O(1)$.
\end{theorem}

\begin{proof}
Fix a segment $k$ and condition on $\eta_k=\delta\neq 0$. Define the unmutated, base post-jump gap as
\be
c(\delta) = G^{(k)}(-\delta) > 0 .
\ee
Let $Z_i = G^{(k)}(M_k^{(i)}-\delta)$ and $Z_{\min}=\min_{1\le i\le N_{\mathrm m}} Z_i$ (conditional on $\eta_k=\delta$). Because the base is retained, the instantaneous excess loss of the swarm is
\be
L^{\star,(k)}(\mathbf{M}_k,\delta) = \min\{c(\delta),Z_{\min}\} .
\ee
Pick any $\alpha\in(0,1)$ (e.g.\ $\alpha=1/2$). By continuity of $G^{(k)}$ at $0$ and $G^{(k)}(0)=0$, there exists a radius $r_\alpha>0$ such that
\be
\label{eq:G_bound}
\|u\|<r_\alpha \Rightarrow G^{(k)}(u) < \alpha c(\delta) .
\ee
Therefore, if a mutated model lands within that ball around the new optimum $\delta$, we have
\be
\|M_k^{(i)}-\delta\|<r_\alpha \Rightarrow G^{(k)}\!\big(M_k^{(i)}-\delta\big) < \alpha c(\delta) < c(\delta) ,
\ee
i.e.\ it strictly improves on the baseline model. Because each $M_k^{(i)}$ has full support and is independent of $\delta$, the event $\|M_k^{(i)}-\delta\|<r_\alpha$ has positive probability for any $i\in\{1,\dots,N_{\mathrm m}\}$:
\be
\label{eq:p_alpha}
p_\alpha(\delta) = \Pr\left(\|M_k^{(i)}-\delta\|<r_\alpha\right) > 0 .
\ee
With $N_{\mathrm m}$ independent mutated models, the probability that at least one improves is then
\be
\Pr\left(\min_{1\le i\le N_{\mathrm m}} \|M_k^{(i)}-\delta\|<r_\alpha\right)
= 1-\big(1-p_\alpha(\delta)\big)^{N_{\mathrm m}} ,
\ee
which is strictly positive for every $N_{\mathrm m}\ge 1$ and increases to $1$ as $N_{\mathrm m}\to\infty$.

We can now bound the expected post-jump gap for the swarm of mutated models. Using $\min\{c(\delta),Z_{\min}\}=c(\delta)-[c(\delta)-Z_{\min}]_+$, where $x_+=\max\{x,0\}$, we have
\be
\label{eq:L_star_delta}
\mathbb E\left[L^{\star,(k)}(\mathbf{M}_k,\delta)\middle|\eta_k=\delta\right] = c(\delta) - \mathbb E\left[(c(\delta)-Z_{\min})_+ \middle|\eta_k=\delta\right] < c(\delta) ,
\ee
because $\Pr(Z_{\min}<c(\delta)\mid \eta_k=\delta)>0$ from \autoref{eq:p_alpha}. 

Since \autoref{eq:L_star_delta} holds for all $\delta\neq 0$ and the losses are bounded (resulting in finite expectations), we can uncondition by integrating both sides over $\mathbb P_{\eta_k}(\eta_k)$ and obtain
\be
\label{eq:rho_N_m}
\rho_{N_{\mathrm m}} = \mathbb E\left[L^{\star,(k)}(\mathbf{M}_k,\eta_k)\right] = \int \mathbb E\left[L^{\star,(k)}(\mathbf{M}_k,\delta)\middle|\eta_k=\delta\right] d\mathbb P_{\eta_k}(\delta) < \int c(\delta) d\mathbb P_{\eta_k}(\delta) = \mathbb E\left[ G^{(k)}(-\eta_k)\right] = \rho_\mathrm{base} .
\ee
Non-negativity of $\rho_{N_{\mathrm m}}$ follows because $G^{(k)}\geq 0$ and $L^{\star,(k)}$ is a minimum of non-negative terms. This way, we obtain the desired inequality:
\be
0 \le \rho_{N_{\mathrm m}} < \rho_\mathrm{base} .
\ee
Monotonicity in $N_{\mathrm m}$ is immediate since $Z_{\min}$ is a minimum over a larger set as $N_{\mathrm m}$ grows. We can express this formally by defining
\be
Z_{\min}^{(N)}=\min_{1\le i\le N} G^{(k)}\left(M_k^{(i)}-\delta\right) ,
\qquad
L_{N}^{\star,(k)}(\mathbf{M}_k,\delta)=\min\left\{c(\delta),Z_{\min}^{(N)}\right\} .
\ee
Adding one more mutated model gives
\be
Z_{\min}^{(N+1)}=\min\left\{Z_{\min}^{(N)},G^{(k)}(M_k^{(N+1)}-\delta)\right\}\le Z_{\min}^{(N)} .
\ee
As a result, we obtain a pointwise argument for monotonicity in $N_{\mathrm m}$:
\be
L_{N+1}^{\star,(k)}(\mathbf{M}_k,\delta)=\min\{c(\delta),Z_{\min}^{(N+1)}\}
\le \min\{c(\delta),Z_{\min}^{(N)}\}=L_{N}^{\star,(k)}(\mathbf{M}_k,\delta) .
\ee
Taking expectations (first conditional on $\eta_k=\delta$, then unconditioning) yields
\be
\rho_{N_{\mathrm m}+1} = \mathbb E\left[L_{N+1}^{\star,(k)}(\mathbf{M}_k,\eta_k)\right]
\le \mathbb E\left[L_{N}^{\star,(k)}(\mathbf{M}_k,\eta_k)\right] = \rho_{N_{\mathrm m}} ,
\ee
with strict inequality under full support and $\Pr(\eta_k=0)=0$.

Having demonstrated monotonicity in $N_{\mathrm m}$, we can now show that $\rho_{N_{\mathrm m}}\to 0$ as $N_{\mathrm m}\to\infty$. For any $\varepsilon>0$, by continuity of $G^{(k)}$ choose $r_\varepsilon>0$ with $\|u\|\le r_\varepsilon \Rightarrow G^{(k)}(u)\le \varepsilon$ as in \autoref{eq:G_bound}. Full support gives
\be
\Pr\left(Z_{\min}\le \varepsilon \middle|\eta_k=\delta\right) = 1-\big(1-p_\varepsilon(\delta)\big)^{N_{\mathrm m}} \xrightarrow[N_{\mathrm m}\to\infty]{} 1 ,
\ee
where $p_\varepsilon(\delta)=\Pr(\|M_k^{(i)}-\delta\|\le r_\varepsilon)>0$ for any $i\in\{1,\dots,N_{\mathrm m}\}$. The expected post-jump gap is bounded by the weighted sum of the two terms in the definition of $L^{\star,(k)}(\mathbf{M}_k,\delta)$, where the weights are the probabilities of the events that the mutated model does or does not land within the ball of radius $r_\varepsilon$ around the new optimum. As $N_{\mathrm m}\to\infty$, the first weight tends to unity and the second weight tends to zero, i.e.
\be
\mathbb E\left[L^{\star,(k)}(\mathbf{M}_k,\delta)\middle|\eta_k=\delta\right]
\le \varepsilon\cdot \Pr(Z_{\min}\le\varepsilon\mid\eta_k=\delta)
+ c(\delta)\cdot \Pr(Z_{\min}>\varepsilon\mid\eta_k=\delta)
\xrightarrow[N_{\mathrm m}\to\infty]{} \varepsilon .
\ee
Therefore, unconditioning over $\eta_k$ yields
\be
\rho_{N_{\mathrm m}} = \mathbb E\left[L^{\star,(k)}(\mathbf{M}_k,\eta_k)\right] = \int \mathbb E\left[L^{\star,(k)}(\mathbf{M}_k,\delta)\middle|\eta_k = \delta\right] d\mathbb P_{\eta_k}(\delta) \le \varepsilon .
\ee
Since $\varepsilon>0$ is arbitrary, we obtain $\lim\limits_{N_{\mathrm m}\to\infty}\rho_{N_{\mathrm m}}=0$.

Finally, we can now show that the expected jump contribution to $\mathbb E[\mathcal{R}_T(h_{1:T})]$ is strictly below the single-model contribution. Let $N_T$ be the number of jumps up to time $T$. By the same argument as in \autoref{thm:linear_regret}, $\mathbb E[N_T]=\lambda T+O(1)$. Each jump contributes an expected post-jump gap $\rho_{N_{\mathrm m}}$, so by linearity of expectation
\be
\mathbb E\left[\sum_{k:\tau_k\le T} L^{\star,(k)}(\mathbf{M}_k,\eta_k)\right] = \mathbb E[N_T]\rho_{N_{\mathrm m}} = (\lambda T)\rho_{N_{\mathrm m}} + O(1) < (\lambda T)\rho_\mathrm{base} + O(1) .
\ee
This is the jump contribution in $\mathbb E[\mathcal R_T(h_{1:T})]$. The inequality follows directly from $\rho_{N_{\mathrm m}}<\rho_\mathrm{base}$ (\autoref{eq:rho_N_m}). Adding the within-segment adaptation term $\bar a\lambda T+O(1)$ on both sides as in \autoref{thm:linear_regret} yields the overall bound
\be
\mathbb E\left[\mathcal R_T(\text{best-of-swarm})\right] \le (\rho_{N_{\mathrm m}}+\bar a)\lambda T + O(1) < (\rho_\mathrm{base}+\bar a)\lambda T + O(1) .
\ee
Here, `best-of-swarm' denotes the process that, at each time $t$, takes the minimum instantaneous risk across the $N_{\mathrm m}$ models. By definition of this process, the expectation is upper-bounded because taking a pointwise minimum across models cannot increase the within-segment adaptation term relative to any single model, but could strictly decrease it. This completes the proof of \autoref{thm:swarm_gap}.
\end{proof}

\autoref{thm:swarm_gap} formalizes the Flawed-in-Nature mechanism, i.e.\ ex-ante parameter diversity reduces the expected post-jump gap. Combining this with the decomposition in the proof of \autoref{thm:linear_regret} yields the swarm analogue of the single-model bound. Most importantly, the argument relies only on bounded losses (for integrability), continuity of the excess-risk profile at the optimum, non-degenerate jumps, and mutation increments that are independent of the preceding data, mean-zero, independent across models, and with full support. It is agnostic to the details of within-segment training and thus applies generally.

We emphasize that the full-support condition is an asymptotic guarantee. In any finite segment, the cumulative mutation vector is initially concentrated near the pre-jump parameter values. The effective coverage depends on the dimensionality $p$, the mutation rate $\mu$, and the expected segment length $1/\lambda$. In the experimental setting of \S\ref{sec:experiments} (which adopts $p=4$, $\mu=0.1$, $\lambda=0.1$), the expected number of mutations per coefficient per segment is $\mu/(p\lambda) \approx 0.25$, implying that most coefficients of a given model have not mutated within a typical segment. The practical benefit therefore relies on the mutation-induced spread being commensurate with the typical jump magnitude, rather than on covering the full parameter space. This is satisfied when the mutation drift rate matches the environmental drift rate (\S\ref{sec:optimization}).

In \S\ref{sec:inference} we provide a practical inference-synthesis layer that, in expectation, reduces regret relative to the unmutated base model by allocating more weight to the best-performing members in a swarm of mutated models.

\section{Inference Synthesis} \label{sec:inference}
As posited in \S\ref{sec:intro}, the evolution of intelligence in nature has required the combination of heritable variation (expressed as model mutation in \S\ref{sec:mutation_def}) and natural selection \citep{darwin59}. In this paper, we follow an analogous approach. The performance benefit of the Flawed-in-Nature mechanism is realized by synthesizing the intelligence generated by the model swarm into a single inference. The key difference to natural selection is that we do not apply a binary (hard) selection function to control the contribution of different models to the swarm, but instead aggregate the model outputs using a weighted average (soft selection).

Many examples of inference synthesis mechanisms exist, in centralized and decentralized forms \citep[e.g.][]{jacobs91,gneiting13,yao18,carvalho23}. Here we use the mechanism of \citet{kruijssen24}, which describes the Allora Network, a decentralized model coordination network. In its original form, this mechanism synthesizes a swarm's outputs via a context-aware linear pool with dynamic weights derived from forecasted regrets \citep[also see e.g.][]{yao18,pfeffer25}. For the purpose of this paper, we omit regret forecasting and instead set weights by applying a scaled logistic gate to an exponential moving average (EMA) of historical regrets \citep[also see e.g.][]{catania18}. This is a saturating alternative to exponentially weighted average forecasting \citep[EWAF, e.g.][]{cesa06} that is closely related to the Hedge algorithm of \citet{freund97b} and sits at the intersection of linear predictive pooling \citep[e.g.][]{gneiting13,yao18} and online expert aggregation. The post-jump regime switching central to the Flawed-in-Nature mechanism has a structural parallel to the sleeping experts framework \citep{freund97}, where the identity of the best expert changes over time \citep[see also][Chapter~5]{cesa06}. We first summarize the key aspects of the adopted, simplified mechanism, before demonstrating that it successfully captures the advantage offered by mutating model swarms.

\subsection{Summary of the Adopted Inference Synthesis Mechanism} \label{sec:mechanism}
Here we provide a brief summary of the adopted inference synthesis mechanism. For a more detailed description that includes regret forecasting, see \citet{kruijssen24}.

Assume an online data stream that is used to generate periodic inferences for a number of time steps or epochs $N_\mathrm{e}$, with $i\in\{1,\dots,N_\mathrm{e}\}$. We consider a system of $N_{\mathrm m}$ models, where each model $j\in\{1,\dots,N_{\mathrm m}\}$ produces a scalar inference $I_{ij}$ using its own dataset $\mathcal{D}_{ij}$ and model $h_{ij}$, i.e.
\be
I_{ij} = h_{ij}(\mathcal{D}_{ij}) .
\ee
The inference synthesis mechanism is a linear pool with dynamic weights that define the synthesized inference as
\be
\label{eq:I_i}
I_i = \frac{\sum_j w_{ij} I_{ij}}{\sum_j w_{ij}} .
\ee
The weights $w_{ij}$ are derived from the historical performance of the individual models. They are calculated at the beginning of each epoch $i$ and are given by
\be
w_{ij} = g(\hat{\mathcal{R}}_{i-1,j}) ,
\ee 
where $w_{ij}$ is the weight of the $j$-th model during epoch $i$, $\hat{\mathcal{R}}_{i-1,j}$ is the normalized historical regret of the network relative to the $j$-th model up to the end of epoch $i-1$, and $g(x)$ is the regret-to-weight mapping function, which is given by
\be
\label{eq:g_x}
g(x) = \frac{p}{\mathrm{e}^{-p(x-c)}+1} ,
\ee
where $p$ sets the normalization and slope of the mapping function, and $c$ sets its knee. For the purpose of this paper, we adopt $p=3$ and $c=0.75$ (see \citealt{kruijssen25} for details). This function results in a monotonically increasing weight with the network regret. The regret is normalized by the standard deviation across all models to control the distribution of weights near the knee of the mapping function, i.e.
\be
\label{eq:regret_norm}
\hat{\mathcal{R}}_{i-1,j} = \frac{\mathcal{R}_{i-1,j}}{\sigma_{i-1}+\epsilon} ,
\ee
where $\sigma_{i-1}=\mathrm{std}(\mathcal{R}_{i-1,1},\dots,\mathcal{R}_{i-1,N_{\mathrm m}})$ is the standard deviation of the historical regrets across all $N_{\mathrm m}$ models up to the end of epoch $i-1$, and $\epsilon=0.01$ is a small constant to avoid division by zero. The regret is defined as an EMA of the instantaneous excess loss, i.e.
\be
\label{eq:regret}
\mathcal{R}_{ij} = \alpha(\log\mathcal{L}_i - \log\mathcal{L}_{ij}) + (1-\alpha) \mathcal{R}_{i-1,j} ,
\ee
where $\mathcal{L}_i$ is the instantaneous loss of the synthesized inference $I_i$, $\mathcal{L}_{ij}$ is the instantaneous loss of the $j$-th model's inference $I_{ij}$, and $\alpha\in(0,1]$ is the EMA parameter. We adopt $\alpha = 0.1$ as the fiducial value, which provides a reasonable balance between historical performance and recency \citep{kruijssen25}. This log-loss formulation requires strictly positive losses, which is satisfied under common loss functions when predictions are not exact.

The above system of equations means that the weight of a model is a scaled logistic gate in the standardized log-loss difference between the synthesized inference and the $j$-th model's inference. It generates a power-law relation between weight and loss for $\hat{\mathcal{R}}\ll c$ (where $g(x)\propto \mathrm{e}^{p(x-c)} \Rightarrow w_{ij}\propto \mathcal{L}_{ij}^{-\beta}$ with $\beta=p/(\sigma_{i-1}+\epsilon)$) and a knee at $\hat{\mathcal{R}}=c$ (with $w_{ij}\rightarrow p$ as $\hat{\mathcal{R}}\rightarrow\infty$). Losses $\mathcal{L}_i$ and $\mathcal{L}_{ij}$ are realized at the beginning of the next epoch $i+1$, allowing $\hat{\mathcal{R}}_{i,j}$ to be updated then and weights $w_{i+1,j}$ to use $\hat{\mathcal{R}}_{i,j}$.

The full inference synthesis mechanism described in \citet{kruijssen24} additionally considers context-aware `forecast-implied inferences', which use forecasted regrets to set weights and are then added to the pool of model inferences. Each forecast-implied inference incorporates information about the expected performance of the models under the current conditions. While clearly beneficial in live systems, the current paper does not require this additional layer of complexity. Therefore, we omit performance forecasting and rely exclusively on historical performance to set weights. In short, our simplified variant is a context-aware linear pool \citep{gneiting13} with online, regret-based weights in the spirit of EWAF \citep{cesa06}. In the full system \citep{kruijssen24}, the use of forecasted regrets corresponds to input-dependent stacking \citep{yao18}.

\subsection{Flawed-in-Nature Advantage Captured by Inference Synthesis} \label{sec:advantage}
The Flawed-in-Nature argument in \S\ref{sec:proof} shows that, after an environmental jump, a model swarm exhibiting population diversity is likely to contain at least one model with instantaneous loss close to the Bayes risk. In \S\ref{sec:mechanism}, we then describe how Allora's inference synthesis layer forms a weighted linear pool with weights $w_{ij}=g(\hat{\mathcal R}_{i-1,j})$ (i.e.\ a scaled logistic gate on an EMA of the standardized network regret). We now quantify why this synthesis is sufficient to inherit the advantage introduced by the Flawed-in-Nature mechanism.

\autoref{thm:swarm_gap} states that, for a swarm with $N_{\mathrm m}$ mutating models, the expected post-jump excess loss satisfies $0\le\rho_{N_m}<\rho_\mathrm{base}$ with $\rho_{N_{\mathrm m}}\downarrow0$ as $N_{\mathrm m}\to\infty$. After an environmental jump, it is therefore likely that at least one mutated model lands close to the new Bayes optimum. The inference synthesis mechanism described in \S\ref{sec:mechanism} turns this model population diversity into performance via three steps. First, the `lucky' mutated model's instantaneous loss $\mathcal{L}_{ij}$ drops as a result of the environmental jump. Secondly, the regret in \autoref{eq:regret} increases the model's advantage relative to the network $\mathcal{R}_{ij}$ via their $\log\mathcal{L}$ differences. Thirdly, the logistic gate $g(x)$ in \autoref{eq:g_x} assigns a larger weight $w_{ij}$ to the lucky mutated model, so it contributes more strongly in the linear pool of \autoref{eq:I_i}. The pooled predictor $I_i$ then tends towards this best available model.

We can formalize the above intuition as follows. As defined in \S\ref{sec:inference}, $\mathcal{L}_{ij}=\ell(I_{ij}, y_i)$ is the loss of model $j$ at epoch $i$ and $\mathcal{L}_i=\ell(I_i,y_i)$ is the loss of the synthesized inference. Additionally, we define the normalized weight as
\be
\label{eq:norm_weight}
\hat{w}_{ij}=\frac{w_{ij}}{\sum_k w_{ik}}.
\ee
Consistently with \S\ref{sec:proof}, we define the loss of the best-performing model as $\mathcal{L}_i^\star=\min_j \mathcal{L}_{ij}$. By \citet{jensen06}'s inequality, the convexity of the evaluation loss (true for the mean squared error loss used below) implies the bound
\be
\label{eq:norm_weight_bound}
\mathcal{L}_i\leq\sum_j\hat{w}_{ij}\mathcal{L}_{ij} ,
\ee
from which it follows that the synthesized inference's excess loss is bounded by
\be
\label{eq:norm_weight_bound_2}
\mathcal{L}_i-\mathcal{L}_i^\star\leq\sum_{j\neq \star}\hat{w}_{ij}(\mathcal{L}_{ij}-\mathcal{L}_i^\star) .
\ee
If we then define the worst gap as $G_i^{\max}=\max_j\big(\mathcal{L}_{ij}-\mathcal{L}_i^\star\big)\geq 0$, we can upper-bound the above by
\be
\label{eq:norm_weight_bound_3}
\mathcal{L}_i-\mathcal{L}_i^\star\leq\sum_{j\neq \star} \hat{w}_{ij}G_i^{\max} = (1-\hat{w}_{i\star})G_i^{\max} .
\ee
This bound can equivalently be expressed as a bound on the regret update term in \autoref{eq:regret} (i.e.\ a bound on the log-loss difference), by defining $r_{ij}=\mathcal{L}_{ij}/\mathcal{L}_i^\star$, with $r_{i\star}=1\leq r_{ij}$ for all $j\neq \star$, and $r^{\max}_i=\max_j r_{ij}$. Then we can use $\sum_{j\neq \star} \hat{w}_{ij}r_{ij} \leq (1-\hat{w}_{i\star})r^{\max}_i$ to obtain
\be
\label{eq:norm_weight_bound_4}
\log\mathcal{L}_i-\log\mathcal{L}_i^\star \leq \log{\left\{1+(1-\hat{w}_{i\star})(r^{\max}_i-1)\right\}} \leq (1-\hat{w}_{i\star})(r^{\max}_i-1) .
\ee

Both of these bounds show that the excess loss of the synthesized inference is bounded by the worst gap, scaled by the normalized weights of the other models. In other words, the only way the synthesized inference can lag the best model is through the allocation penalty $(1-\hat{w}_{i\star})$. As $\hat{w}_{i\star}\uparrow1$, the allocation penalty vanishes and the synthesized inference tracks the best model at epoch $i$.

The question then becomes whether (and how quickly) the allocation penalty may become small in practice. To answer this, we need to understand how the logistic gate $g(x)$ behaves in conjunction with the EMA on the regrets. We first define the standardized advantage $d_i$ of the best model relative to its closest competitor as
\be
\label{eq:std_advantage}
d_i = \hat{\mathcal{R}}_{i-1,\star}-\max_{j\ne\star}\hat{\mathcal{R}}_{i-1,j} \equiv \hat{\mathcal{R}}_{i-1,\star}-m_i \geq 0 ,
\ee
where $\hat{\mathcal{R}}_{i-1,\star}$ is the standardized network regret of the best model at epoch $i-1$ as defined in \autoref{eq:regret_norm}, and $m_i=\max_{j\ne\star}\hat{\mathcal{R}}_{i-1,j}$ is its maximum across all other models. Because $g(x)$ is monotonically increasing in the standardized regret, the normalized weight of the best model
\be
\label{eq:norm_weight_monotone}
\hat{w}_{i\star} = \frac{g(\hat{\mathcal{R}}_{i-1,\star})}{\sum_k g(\hat{\mathcal{R}}_{i-1,k})} \geq \frac{g(m_i+d_i)}{g(m_i+d_i)+(N_{\mathrm m}-1)g(m_i)} ,
\ee
is also a monotonically increasing function of $d_i$. If we now define $a_i=\mathrm{e}^{-p(m_i-c)}$, we can rewrite the above to obtain an upper bound on the allocation penalty:
\be
\label{eq:norm_weight_monotone_2}
1-\hat{w}_{i\star} \leq \frac{(N_{\mathrm m}-1)(1+a_i\mathrm{e}^{-p d_i})}{(1+a_i)+(N_{\mathrm m}-1)(1+a_i\mathrm{e}^{-p d_i})} .
\ee

Two regimes now enable the synthesized inference's excess loss from \autoref{eq:norm_weight_bound_2} to become small. In Regime (A), a single (mutated) model clearly outperforms post-jump. In this case, the allocation penalty $(1-\hat{w}_{i\star})$ decays to a small value exponentially on the EMA timescale. In Regime (B), the best model does not exhibit a clear advantage post-jump and the network regrets of the other models cluster near its value. In this case, the allocation penalty does not tend to zero, but the worst gap $G_i^{\max}$ will be small.

We first consider Regime (A), in which a single (mutated) model clearly outperforms post-jump. In that case, its standardized advantage is large ($d_i\gg0$), its standardized network regret is positive ($\hat{\mathcal{R}}_{i-1,\star}>0$), and the standardized network regret of the other models is negative ($m_i\ll0$). This results in $a_i\mathrm{e}^{-p d_i}\to0$ if the best model's standardized regret exceeds the knee in $g(x)$, i.e.\ $\hat{\mathcal{R}}_{i-1,\star}>c$, allowing us to decompose the allocation penalty into a plateau term and a transient decay term as
\be
\label{eq:penalty-plateau-split}
1-\hat{w}_{i\star} \leq \frac{(N_{\mathrm m}-1)\,\big(1+a_i\,\mathrm e^{-p d_i}\big)}{(1+a_i) + (N_{\mathrm m}-1)\,\big(1+a_i\,\mathrm e^{-p d_i}\big)} = \underbrace{\frac{N_{\mathrm m}-1}{N_{\mathrm m}+a_i}}_{\text{plateau}} + \underbrace{\Phi_i}_{\text{decay}} ,
\ee
where the plateau term tends to a small constant since $a_i$ is large for $m_i\ll0$ and fixed $N_{\mathrm m}$ (unless $N_{\mathrm m}$ is very large, in which case dilution effects may dominate and the plateau term increases). The transient decay term $\Phi_i$ exponentially decays under the EMA of \autoref{eq:regret} as $a_i\mathrm{e}^{-p d_i}\to0$:
\be
\label{eq:penalty-transient}
\Phi_i = \frac{(N_{\mathrm m}-1)\,a_i\,(a_i+1)\,\mathrm e^{-p d_i}}{\big(N_{\mathrm m}+a_i\big)\big(N_{\mathrm m}+a_i+(N_{\mathrm m}-1)\,a_i\,\mathrm e^{-p d_i}\big)} \approx \frac{(N_{\mathrm m}-1)\,a_i\,(a_i+1)}{(N_{\mathrm m}+a_i)^2}\,\mathrm e^{-p d_i} .
\ee
The timescale for the exponential decay of $\Phi_i$ is derived by considering the update of the advantage $d_i$ in \autoref{eq:std_advantage} under the EMA of \autoref{eq:regret}. The regrets are calculated through an EMA of the log-loss differences, so initially the advantage grows linearly at a rate
\be
\label{eq:advantage-growth}
\gamma = d_{i+1}-d_i \approx \frac{\alpha}{\sigma_i+\epsilon}\,\overline{D}_i ,
\ee
where $\overline{D}_i = \min_{j\ne\star}\big(\log \mathcal{L}_{ij} - \log \mathcal{L}_{i\star}\big)\geq0$ is the log-loss difference between the best and second-best model at epoch $i$. The approximation in the final term assumes that $d_i\ll\overline{D}_i/(\sigma_i+\epsilon)$, i.e.\ the best model's standardized advantage is small compared to the standardized log-loss difference used in the EMA update. This is typically satisfied shortly after a meaningful jump.

For the purpose of the decay timescale calculation, we now assume that $\sigma_i$ and $\overline{D}_i$ vary slowly immediately after a jump and treat them as locally constant, i.e.\ $\sigma_i\approx\sigma$ and $\overline{D}_i\approx\overline{D}$. Combining this with \autoref{eq:penalty-transient} allows us to obtain the half-life $t_{1/2}$ of the allocation penalty from the exponential decay term $\mathrm{e}^{-p d_i}$ as
\be
\label{eq:decay-rate}
t_{1/2} = \frac{\ln 2}{p\gamma} = \frac{(\sigma+\epsilon)\,\ln 2}{p\,\alpha\,\overline{D}} .
\ee
This expression shows that the timescale for concentrating the weight on the best model shortens for a steeper logistic gate (larger $p$), a faster EMA (larger $\alpha$), or a larger post-jump advantage relative to the standard deviation of the network regrets (larger $\overline{D}/\sigma$). In practice, we expect $\epsilon\ll\sigma$, so that (for fiducial parameters) the half-life is approximately
\be
\label{eq:half-life-approx}
t_{1/2} \approx \frac{\sigma\,\ln 2}{p\,\alpha\,\overline{D}} \approx 2.31\frac{\sigma}{\overline{D}} \quad \text{for } p=3 \text{ and } \alpha=0.1 .
\ee
This means that a rapid adjustment of the synthesized inference to the best model (i.e.\ within a handful of epochs) is expected to occur if the log-loss difference between the best and second-best model is at least of the order of the standard deviation of the network regrets. If immediately after a jump the best-in-loss model does not yet have the largest standardized regret (i.e.\ $d_i<0$), then in the early-time regime $d_i$ still grows at a rate $\approx\gamma$ per epoch and becomes positive within $\approx|d_i|/\gamma$ steps. Once $d_i\ge 0$ (and $\hat{\mathcal R}_{i-1,\star}>c$), the decay analysis above applies. Adding a set of performance forecasting models \citep{pfeffer25} to enable the use of forecasted regrets would reduce the timescale for adjustment.

Finally, we consider the complementary regime (B) wherein the synthesized inference's excess loss may become small. This occurs when the best model does not exhibit a clear advantage post-jump and the network regrets of the other models cluster near its value. In this case, the allocation penalty $(1-\hat{w}_{i\star})$ does not tend to zero, but the worst gap $G_i^{\max}$ will be small. In other words, the small outperformance of the best model means that the synthesized inference cannot be much worse, regardless of the weight distribution across the model pool. In practice, decentralized learning pools may exhibit both regimes after a jump. Initially, the best (mutated) model has a clear advantage, but as the other models adjust to the environmental change, that advantage may shrink and the other models catch up. In the context of the above regimes, this corresponds to a gradual transition from Regime (A) to Regime (B), across which the synthesized inference's excess loss remains small. 

We see that the Flawed-in-Nature mechanism increases the chance of a strong advantage after an environmental change. The logistic gate drives down $(1-\hat w_{i\star})$ exponentially on the EMA timescale and the synthesized inference's loss tends towards the loss of the best model. This means that inference synthesis naturally inherits the swarm's diversity benefit by rapidly tracking the best model after a jump, as long as it offers a notable advantage.

\section{Numerical Experiments Comparing Original and Mutating Model Swarms} \label{sec:experiments}
We now turn to an experimental validation of the Flawed-in-Nature mechanism, its underlying theorems, and its realization during inference synthesis as presented in \S\ref{sec:proof}-\ref{sec:inference}. We consider a simple synthetic regression problem with a changing environment, for which we compare the performance of a single original model, a swarm of original models, and a swarm of mutating models. We show that the swarm of mutating models outperforms any single model, thereby breaking the linear regret bound that constrains individual models.

\subsection{Experiment Design and Description of the Models} \label{sec:design}
We consider a three-variable linear regression problem, where the target variable is generated as a linear combination of three features and a non-zero intercept, as well as added Gaussian noise with standard deviation $\sigma_{\mathrm{noise}}=0.1$, representing the irreducible Bayes risk. Both the features and the coefficients of the linear combination are randomly generated from a $\mathcal N(0,1)$ distribution. During their subsequent evolution, these four coefficients each independently undergo a random walk through a Poisson process with rate $\lambda=1/C_{\mathrm{dt}}$ and magnitude drawn from a normal distribution $\mathcal{N}(0,C_{\mathrm{std}})$. We adopt defaults of $C_{\mathrm{std}}=0.1$ and $C_{\mathrm{dt}}=10$ epochs, chosen to ensure sufficient jump events within the experiment duration. This creates a numerical environment that satisfies the drift model of \S\ref{sec:drift}. When needed, a static environment is created by setting $C_{\mathrm{dt}}\rightarrow\infty$.

The experiments discussed in this section consider the performance of a single original model, a swarm of original models, and a swarm of mutating models. Each model is a ridge regression model with a regularization parameter $\alpha_{\mathrm{ridge}}$ that we draw randomly from a log-uniform distribution over the interval $[10^{0},10^{2.5})$. The training data for the models is constituted by the features used to generate the target variable as described above, as well as the target variable itself. Each epoch, the models are (re-)trained on the entire observation history up to and including the current epoch. Note that the theorem demonstrations below may deviate from this default in select places, in which case we clearly describe and motivate the deviations.

The two parallel swarms of original and mutating models consist of $N_{\mathrm{m}}=32$ models each (matching the Allora network's inference pool size, see \citealt{kruijssen24b}). The `original' (baseline) models are trained as described above, without any mutations. The models in the `mutated' swarm are direct copies of the original models that are subjected to a compounding history of coefficient-level perturbations. This is achieved through a per-coefficient Poisson point process. Analogously to the environmental evolution, we define an overall mutation rate per model (i.e.\ the rate at which a model is expected to undergo a mutation across all coefficients) of $\mu=1/M_{\mathrm{dt}}$, and a per-coefficient mutation magnitude drawn from a normal distribution $\mathcal{N}(0,M_{\mathrm{std}})$. We adopt defaults of $M_{\mathrm{std}}=0.1$ and $M_{\mathrm{dt}}=10$ epochs, chosen to match the environmental evolution, but vary these parameters in \S\ref{sec:optimization}. This mutation rate results in a per-coefficient mutation probability of $q_c=\mu/p$ per epoch as in \autoref{eq:mutations} (where $p$ is the number of coefficients).

For each model, we define a mutation vector $\mathbf{m}_i$ that accumulates over time, where each element $m_{i,c}$ is the cumulative mutation of the $c$-th coefficient (including the intercept). The mutation vector is initialized to zero ($m_{0,c} = 0\ \forall\ c$) and updated at each epoch as
\be
\label{eq:mutation-vector-update}
m_{i,c} = m_{i-1,c} + \varepsilon_{i,c}\varpi_{i,c} ,
\ee
where $\varepsilon_{i,c}\sim\mathcal{N}(0,M_{\mathrm{std}})$ is a random variable drawn from a normal distribution, and $\varpi_{i,c}\in\{0,1\}$ is a binary indicator drawn from a Bernoulli distribution with mutation probability $q_c$, i.e.\ $\varpi_{i,c}\sim\mathrm{Bernoulli}(q_c)$. The mutation vector is then used to obtain the mutated model's coefficient vector $\mathbf{c}_i^\mathrm{m}$ from the original model's coefficient vector $\mathbf{c}_i$ as
\be
\label{eq:mutation-vector-update-2}
\mathbf{c}_i^\mathrm{m} = \mathbf{c}_i + \mathbf{m}_i .
\ee

Inference synthesis is then performed as described in \S\ref{sec:inference}, where the logistic gate parameters are set to $p=3$ and $c=0.75$, and the EMA parameter is set to $\alpha=0.1$ for moderate historical performance tracking. This results in two network inferences, one for the original models ($I_i$) and one for the mutated models ($I_i^\mathrm{m}$). All inferences are compared to the ground truth using a mean squared error (MSE) loss function. A single experiment spans 1000 epochs by default, by which meaningful statistical significance has typically been reached.

Some of the experiments discussed in this section deviate from the above, default setup. In those cases, we describe exactly how the experiments are conducted. This pertains mostly to the theorem demonstrations, which do not always require the full complexity of the default setup.

Our default experimental setup represents a special case for illustrating the much more generally applicable Flawed-in-Nature mechanism. Its simple linear nature allows for a direct comparison between the rate and magnitude of environmental changes ($1/C_{\mathrm{dt}}$ and $C_{\mathrm{std}}$) and model parameter mutations ($1/M_{\mathrm{dt}}$ and $M_{\mathrm{std}}$), which in turn enables a clear way to identify which mutation parameters generate the greatest Flawed-in-Nature advantage. This controlled setting isolates the core mechanism, and extensions to more complex and non-linear settings are discussed in 
§\ref{sec:practical}, where we specifically consider the real-world situation that the rate and magnitude of environmental changes are not known.

\subsection{Simple Experiments Illustrating the Flawed-in-Nature Theorems} \label{sec:results_simple_experiments}
We now turn to a set of simple experiments that illustrate the first three of the Flawed-in-Nature theorems discussed in \S\ref{sec:proof}. While these use some of the setup described in \S\ref{sec:design}, they are greatly simplified to isolate the core mechanisms relevant to each theorem.

\begin{figure}[!t]
  \centering
  \includegraphics[width=0.67\textwidth]{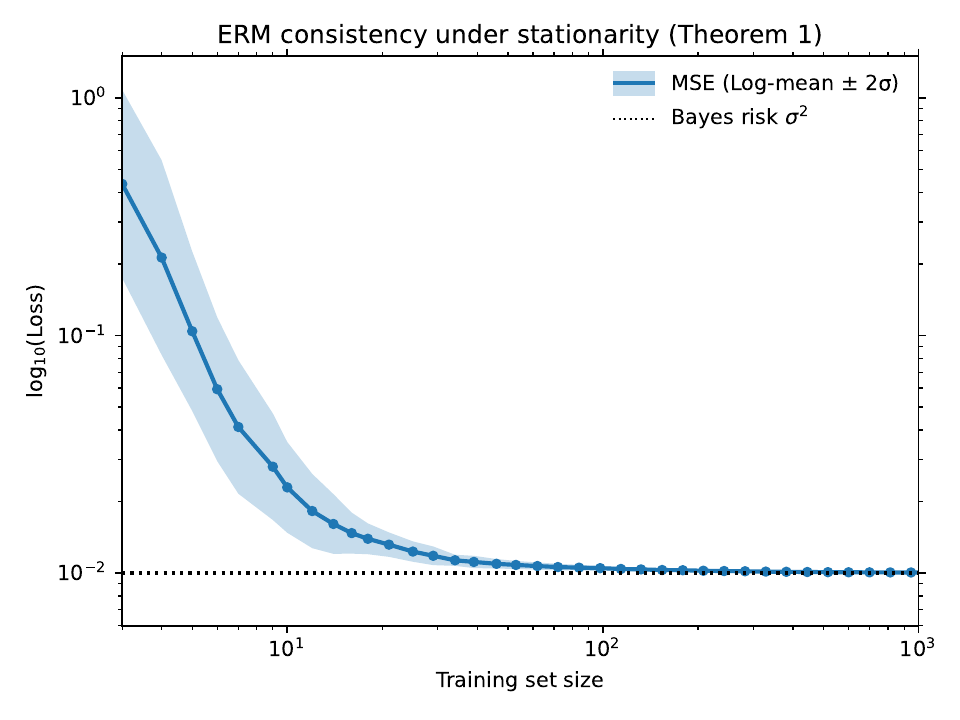}
  \caption{Experimental demonstration of Theorem 1, i.e.\ an empirically-trained model in a static environment reaches the lowest possible loss attainable inside the hypothesis class. Shown is the mean squared error (MSE) loss of the model as a function of the training set size, with the shaded region indicating the $2\sigma$ standard error on the mean across 64 independent runs. The dotted line shows the lowest possible loss attainable inside the hypothesis class, which in this experiment equals the square of the noise in the target variable.}
  \label{fig:theorem1}
\end{figure}
\autoref{thm:erm_consistency} states ERM consistency under stationarity, i.e.\ that as the training set size increases, the model's loss converges to the lowest possible loss attainable inside the hypothesis class. To illustrate this theorem, we modify the setup of \S\ref{sec:design} by first setting the coefficient-change magnitude to zero ($C_{\mathrm{std}}=0$), creating a stationary environment. We then train a single ridge regression model with $\alpha_{\mathrm{ridge}}=0.1$ on increasingly large samples between $3$--$1000$ epochs. For each training set size, we repeat the experiment 64 times and test the model on a held-out test set of 20,000 epochs.

The resulting mean log$_{10}$(MSE) loss curve with $2\sigma$ standard error on the mean is shown in \autoref{fig:theorem1}. As expected, the loss decays towards the irreducible Bayes risk (set by  the noise variance $\sigma_{\mathrm{noise}}^2$) as the training set size increases. This simple experiment illustrates the well-known result that ERM achieves the minimal achievable loss within the hypothesis class under stationarity.

\begin{figure}[!t]
  \centering
  \includegraphics[width=\textwidth]{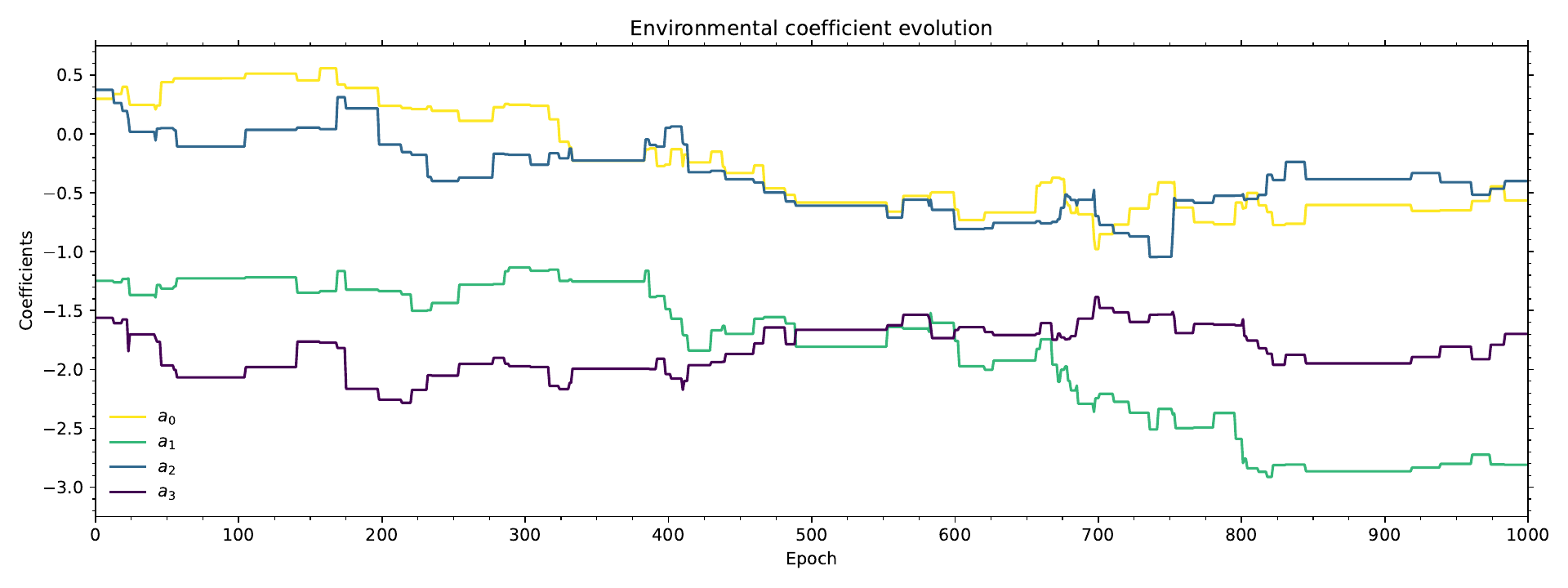}
  \caption{Time evolution of the coefficients defining the environment used in the numerical experiments. The target variable of the experiments is defined by a three-variable linear model with intercept ($a_0$) and slope coefficients ($a_1, a_2, a_3$) that change over time, following a Poisson process with constant rate and magnitude parameters.}
  \label{fig:coefficients}
\end{figure}
In our numerical experiments, a changing environment is created as described in \S\ref{sec:drift}, i.e.\ by perturbing the coefficients of the linear model that defines the environment. For the default parameter setup ($C_{\mathrm{std}}=0.1$ and $C_{\mathrm{dt}}=10$ epochs), the resulting evolution of the coefficients is shown over 1000 epochs in \autoref{fig:coefficients}. The Poisson-driven jumps are clearly visible, with segments of constant coefficients separated by random increments.

Given that the coefficients are drawn from a $\mathcal{N}(0,1)$ distribution, and are thus of order unity, the natural time scale for meaningful environmental changes is of the order of $C_{\mathrm{dt}}/C_{\mathrm{std}} \sim 100$ epochs. This is confirmed visually in \autoref{fig:coefficients}, where the coefficients do not remain stable for more than $\sim200$ epochs.

\begin{figure}[!t]
  \centering
  \includegraphics[width=\textwidth]{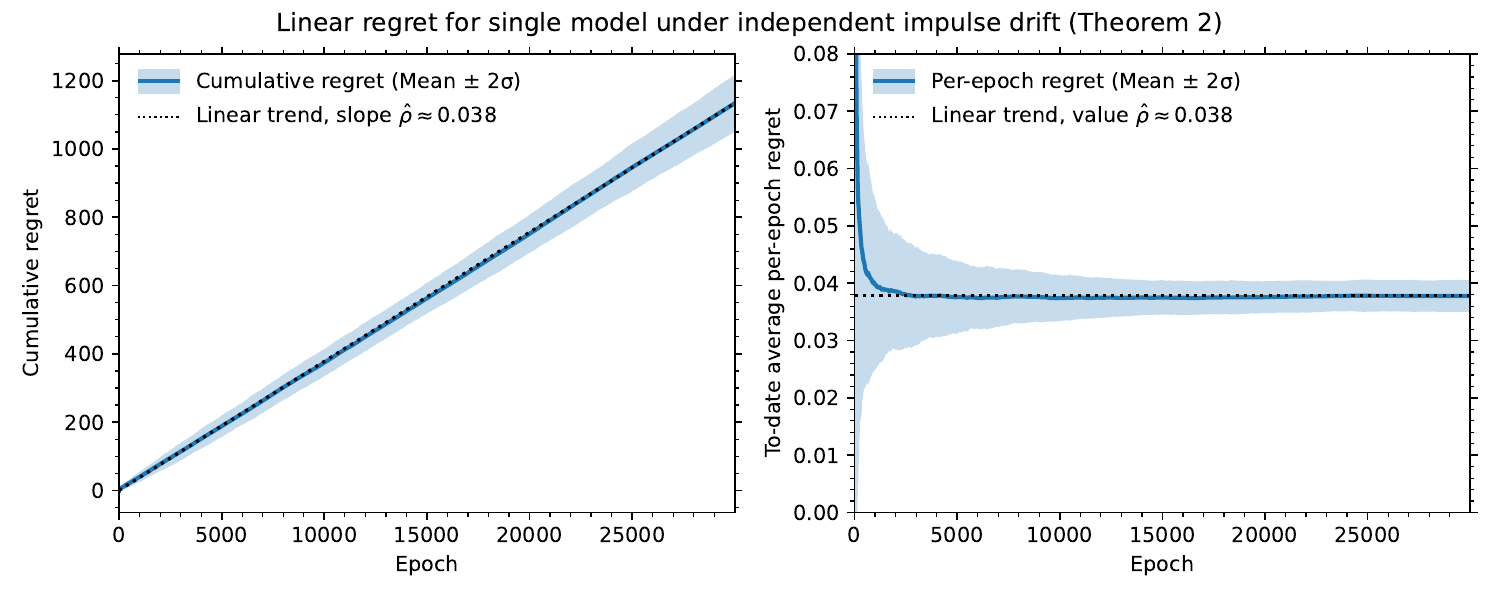}
  \caption{Experimental demonstration of Theorem 2, i.e.\ an empirically-trained model in a changing environment incurs linear regret. Shown are the cumulative regret (left) and the to-date average regret per epoch (right) as a function of the number of epochs. As in \autoref{fig:theorem1}, the shaded region indicates the $2\sigma$ standard error on the mean across (in this case) 64 independent runs. The dotted line shows the linear trend of the cumulative regret.}
  \label{fig:theorem2}
\end{figure}
\autoref{thm:linear_regret} states that an empirically-trained model in a changing environment incurs linear regret. To illustrate this theorem, we use the default setup for a changing environment from \S\ref{sec:design}, but again simplify the experiment by considering only a single ridge regression model (with $\alpha_{\mathrm{ridge}}=0.1$) and test its predictions over a total of 30,000 epochs. Due to the nature of the changing conditions, the training data accumulates information that is irrelevant to the current environment. To maintain a clean experiment and avoid regret biases by stale data, we train the model only on the last 40 epochs of training data at any given time. This number was chosen somewhat arbitrarily to be the same order of magnitude as the mean segment length, and empirically we find it successfully avoids regret biases. The model is retrained this way at each epoch, and each epoch represents an individual test. The Bayes predictor is defined as a linear model using the true coefficients at that epoch. The instantaneous regret is then calculated as the MSE difference between the model's prediction and the Bayes predictor, and the cumulative regret is simply calculated as the sum of the instantaneous regrets up until that time. Finally, we calculate the to-date average regret per epoch as the cumulative regret divided by the epoch number. In order to ascertain statistical significance, the entire experiment is repeated 64 times.

\autoref{fig:theorem2} shows the resulting cumulative regret and average regret per epoch curves, averaged over the 64 realizations, together with the $2\sigma$ standard error on the mean. Dotted lines illustrate the mean linear trend of the cumulative regret. As expected according to \autoref{thm:linear_regret}, the cumulative regret follows the linear trend very closely, and the average regret per epoch stabilizes at a constant value within a few 1000 epochs.

\begin{figure}[!t]
  \centering
  \includegraphics[width=0.67\textwidth]{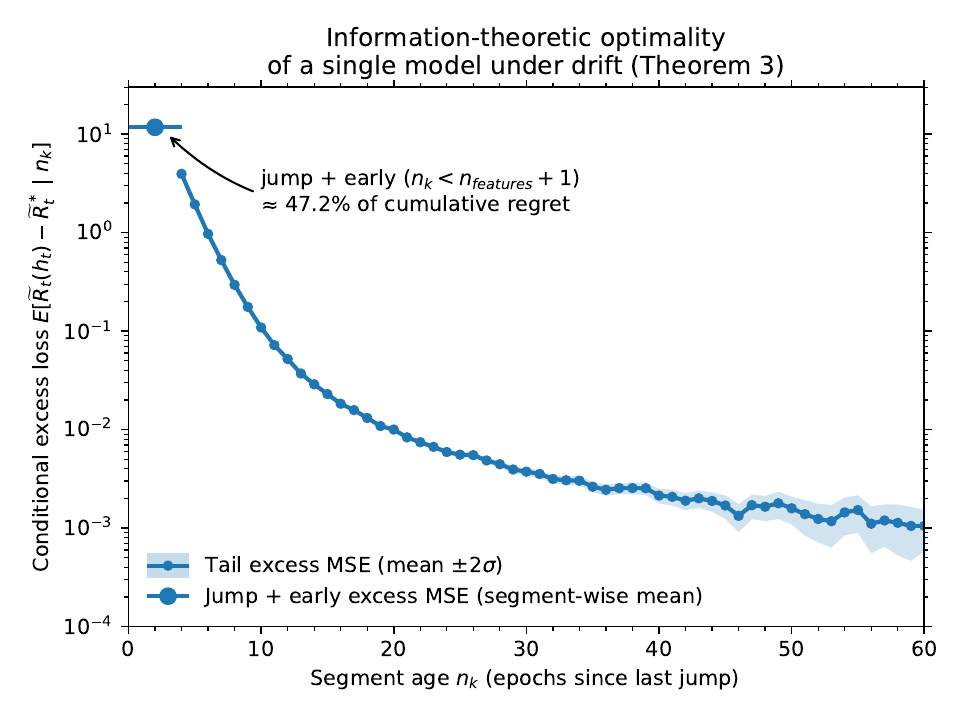}
  \caption{Experimental demonstration of Theorem 3, i.e.\ a single model in a changing environment is still information-theoretically optimal given its access to past observations only, and its excess loss is therefore irreducible. Shown is the conditional excess MSE loss of the model as a function of the number of epochs after a jump (i.e.\ the segment age $n_k$), with the shaded region indicating the $2\sigma$ interval across 64 independent runs. The first four epochs are collected into a single data point, because at $n_k<n_{\mathrm{features}}+1$ there is insufficient data to constrain the model's coefficients. The rapid decay of the conditional excess loss confirms that the single model becomes conditionally Bayes-optimal within each segment.}
  \label{fig:theorem3}
\end{figure}
Furthermore, \autoref{thm:info_optimality} states that even though a model incurs linear regret, it is still information-theoretically optimal given its access to past observations only, and its excess loss is therefore irreducible. To illustrate this theorem, we use the same setup as for \autoref{fig:theorem2}, and specifically evaluate the model performance immediately after an environmental jump (i.e.\ $n_k\geq0$). We again train a ridge regression model with $\alpha_{\mathrm{ridge}}=0.1$ on the data from the latest available segment, i.e. on all data from the previous segment when $n_k=0$, or on the first $n_k$ post-jump observations when $n_k\geq1$. We again compare to the Bayes predictor, which uses the true segment coefficients after the jump. We then calculate the conditional excess loss, i.e.\ the MSE loss difference between the model and the Bayes predictor, which measures how much additional error the model incurs. This gap represents the cost of learning from finite data within each segment, and decays as the model accumulates observations. Within the first $n_{\mathrm{features}}+1$ epochs (in this case 4), the model has insufficient data to constrain its coefficients, and we therefore group the excess losses at these epochs into a single mean value.

The resulting mean conditional excess loss curve with $2\sigma$ standard error on the mean is shown in \autoref{fig:theorem3}. The mean and standard error are taken over the 64 independent runs and all jumps over their 30,000-epoch histories. The curve exhibits rapid decay as the model accumulates observations within the segment, which confirms that within each segment the model converges to the segment-specific Bayes optimum given enough data. Nearly half of the cumulative regret is generated during the first four epochs after a jump. This shows that the irreducible regret from \autoref{thm:linear_regret} and \autoref{fig:theorem2} is concentrated at jump times rather than during stationary periods, and the model itself is still information-theoretically optimal given its access to past observations only.

\subsection{Model Swarm Population Statistics} \label{sec:results_swarm_statistics}
We now turn to an analysis of the full model swarm population of the experiments described in \S\ref{sec:design}. In summary, we consider two swarms of 32 models each. The `original' swarm is trained on the complete historical training data up until the current epoch, and the `mutated' swarm contains copies of the original swarm, but with each model accumulating mutations according to the process described in \S\ref{sec:design}.

\begin{figure}[!t]
  \centering
  \includegraphics[width=\textwidth]{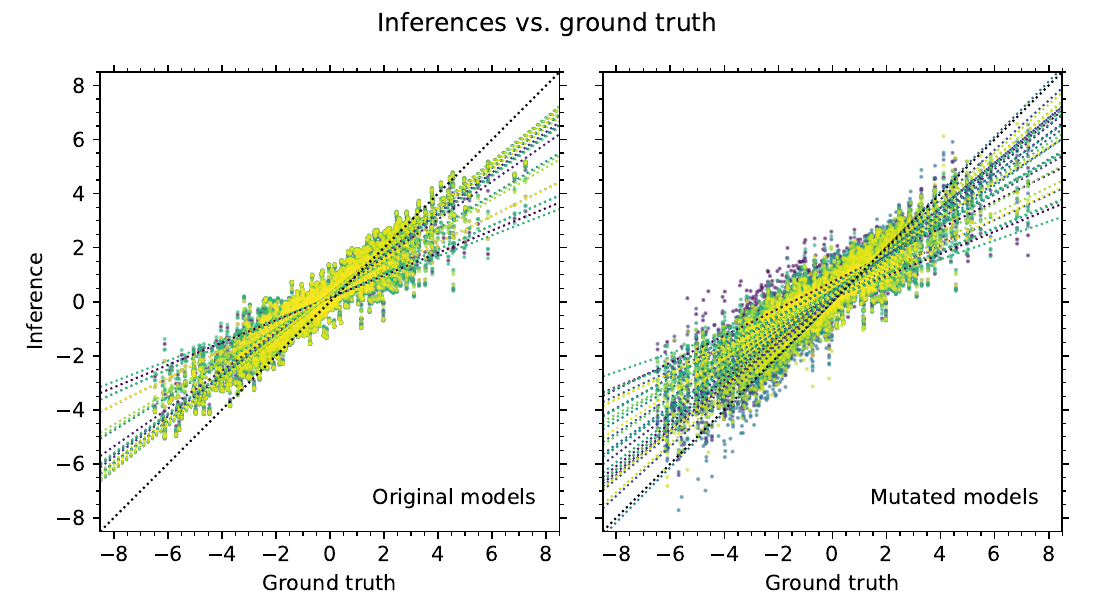}
  \caption{Individual model inferences as a function of the ground truth for the original models (left) and the mutated models (right). Each swarm consists of 32 models, where inferences from the same model are shown in the same color. The black dotted line shows the 1:1 relation, whereas the colored lines show the best-fitting linear regression to the data. The mutated models exhibit more scatter around the 1:1 relation, indicating a more diverse set of inferences. As a result, the best mutated model resides more closely to the 1:1 relation than the best original model.}
  \label{fig:inferences_vs_ground_truth}
\end{figure}
\autoref{fig:inferences_vs_ground_truth} shows the individual inferences of the original (left) and mutated models (right) as a function of the ground truth. Naturally, by perturbing the model coefficients, the mutations introduce a form of noise on the model predictions. This generates a population with greater model diversity, and thus greater scatter around the 1:1 relation. We also see that the best mutated model resides more closely to the 1:1 relation than the best original model. This can be understood as a result of the post-jump gap experienced by the original models. These models are trained on partially obsolete data, and therefore incur linear regret. By contrast, the mutated models have a random chance of being perturbed into a direction that is more favorable to the current environment, and therefore have a better chance of breaking the linear regret bound and aligning with the 1:1 relation.

\begin{figure}[!t]
  \centering
  \includegraphics[width=\textwidth]{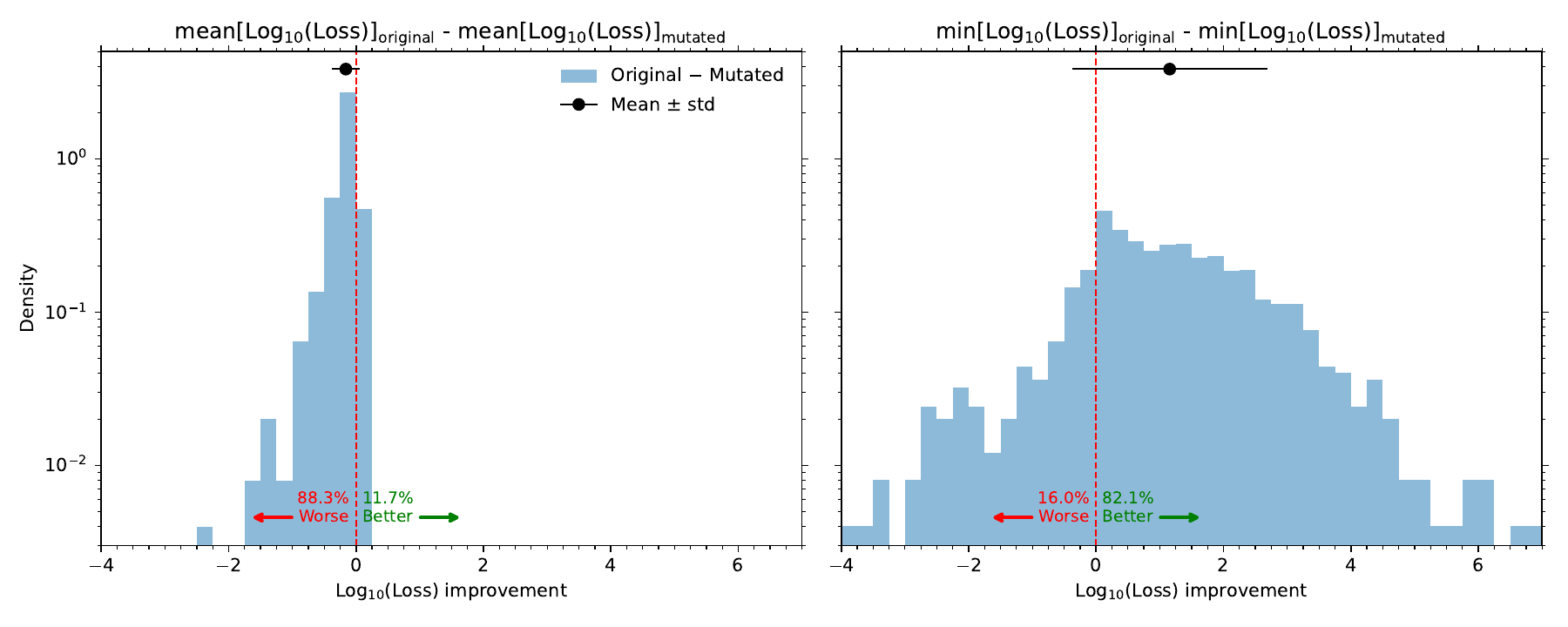}
  \caption{MSE log-loss improvement distributions of the original and mutated models. Shown are histograms of the differences in mean (left) and minimum (right) log-loss of the original and mutated model swarms. The data points with error bars indicate the mean and standard deviation of the improvement distribution. The left-hand panel shows that, on average, the mutated models exhibit worse performance than the original models, and that model mutations are thus not in the best interest of any individual model. However, the right-hand panel shows that the best mutated model typically outperforms the best original model, implying that a suitable inference synthesis mechanism that appropriately allocates weight to the best model will be able to benefit from the Flawed-in-Nature advantage offered by the mutated swarm.}
  \label{fig:loss_improvement}
\end{figure}
We quantify the resulting difference between the mean and minimum log-losses of the original and mutated models in \autoref{fig:loss_improvement}. The figure illustrates how the Flawed-in-Nature mechanism yields individual harm, but collective benefit. In the left-hand panel, we see that the log-loss improvement across all 1000 epochs is expected to be negative on average. This is because the ground truth represents a single point in a high-dimensional space, and there exist more dimensions along which perturbations may drive the model away from the ground truth than those along which they may improve the model's accuracy. Note that the mean is statistically significantly below zero -- the error bars represent standard deviations, hence their large extent. A convincing 88\% of the epochs exhibits a negative mean log-loss improvement. By contrast, the right-hand panel shows that the best mutated model typically outperforms the best original model. In 82\% of the epochs, the lowest log-loss achieved across the original swarm exceeds the lowest log-loss achieved across the mutated swarm. As we will see below, when a sufficient number of models are present, the best model in the mutated swarm will scatter closer to the ground truth than the best model in the original swarm. This means that a performant inference synthesis mechanism should be able to capture this advantage and yield a lower synthesized network loss.

Having considered the broad statistical properties of the original and mutated model swarms, we can now turn to the experimental demonstration of \autoref{thm:swarm_gap}. This theorem states (I) that model diversity lowers the expected post-jump gap, i.e.\ after a jump the expected excess loss (or risk) of the best model in the swarm is lower than the expected excess loss of the best original model, and (II) that the risk of the swarm tends to zero as the number of models approaches infinity. In order to evaluate this theorem, we use the default setup described in \S\ref{sec:design} and detect all epochs at which a jump occurs and the environment changes. For each jump, we then extract the losses of the best original model and the best mutated model at that time.

\begin{figure}[!t]
  \centering
  \includegraphics[width=\textwidth]{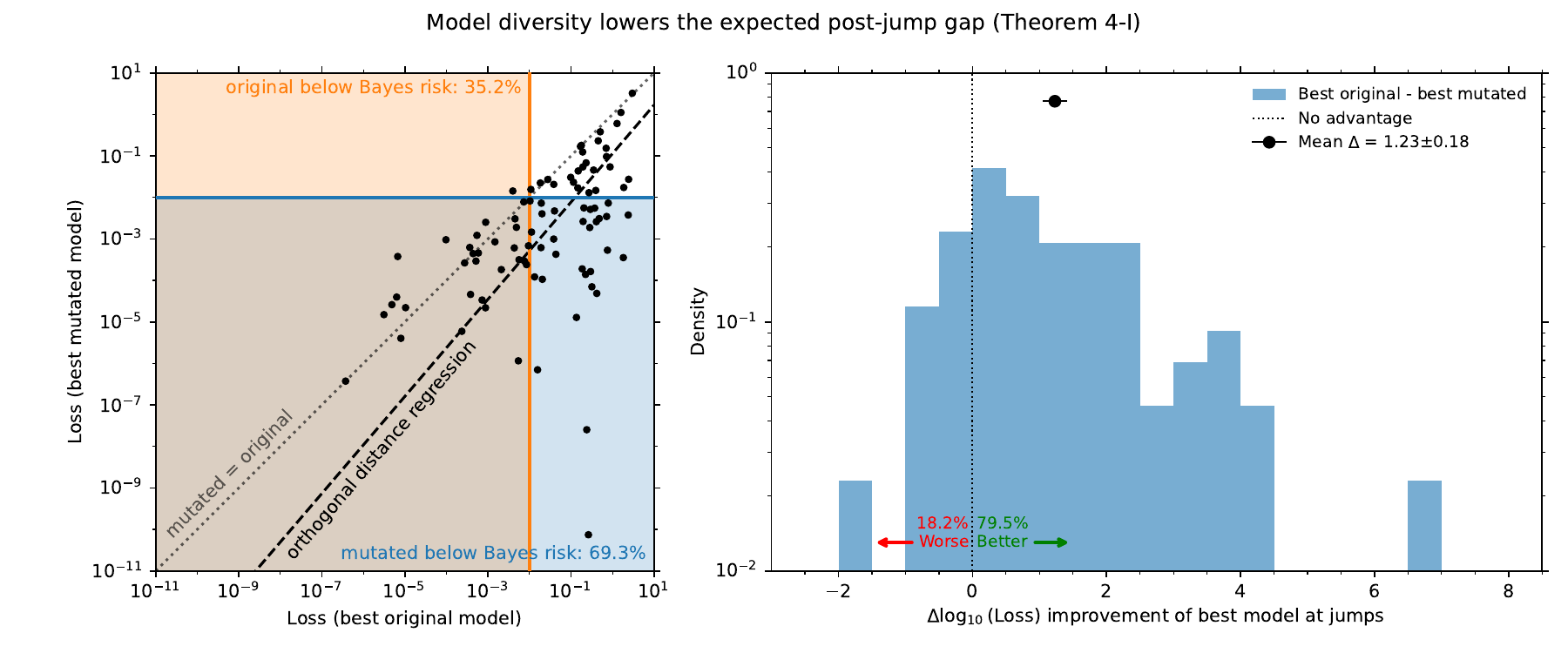}
  \caption{Experimental demonstration of Theorem 4 (I), i.e.\ model diversity breaks the linear regret bound shown in \autoref{fig:theorem2} by reducing the expected post-jump excess loss (`gap'). Left: MSE loss scatter plot of the best mutated model as a function of the best original model. The dashed line shows an orthogonal distance regression to the data. The shaded regions indicate models below the Bayes risk, with annotations indicating the percentage of points for each model swarm. Across all three of these visual indicators, the mutated models typically outperform the original models. Right: Histogram of the MSE log-loss improvement at jump times of the best mutated model relative to the best original model. The data point with error bar indicates the mean and standard error of the mean improvement. The mutated swarm delivers the best model for about 80\% of the jumps.}
  \label{fig:theorem4a}
\end{figure}
The result is shown in \autoref{fig:theorem4a}. In the left-hand panel, we show the losses of the best mutated model as a function of those of the best original model at jump times. We see that most data points lie below the 1:1 line, and indeed an orthogonal distance regression has a slope close to unity, but a considerable offset to lower losses for the mutated models, i.e.\ the best mutated models outperform the best original models on average. Both model swarms have some points below the Bayes risk (this is to be expected as a natural consequence of noise), but the fraction is considerably greater for the mutated models (69.3\% below the Bayes risk) than for the original models (35.2\% below the Bayes risk). The perturbations generated by the mutation process enable the models to incidentally move closer to the ground truth.

Why does this happen? Intuitively, at each jump the new Bayes predictor is an independent random point in parameter space. The original models all form a single cluster around the previous Bayes predictor, whereas the Flawed-in-Nature mechanism generates a cloud of such clusters perturbed away from the previous optimum. Because we pick the best of these mutated models, the probability that at least one mutated model lies very close to the new Bayes predictor (and thus achieves loss near the noise floor) is strictly higher than for the original model pool. This is quantified further in the right-hand panel of \autoref{fig:theorem4a}, which shows a log-loss improvement histogram of the best model as in \autoref{fig:loss_improvement}. However, this time the histogram only considers jump epochs, when the models have not had a chance to adjust to the environmental change. In other words, the improvement shown here exclusively captures the effect of the mutation process. The mean improvement (this time with error bars showing the standard error on the mean) is $6.8\sigma$ above zero, and the mutated swarm hosts the best model in 79.5\% of the jumps (70 out of 88). Note that this outperformance rate is a direct empirical measurement at the tested dimensionality and mutation parameters, and does not rely on the full-support condition of \autoref{thm:swarm_gap} being practically achieved across the full parameter space. These statistics unambiguously illustrate that model diversity does indeed lower the expected post-jump gap.

\begin{figure}[!t]
  \centering
  \includegraphics[width=0.67\textwidth]{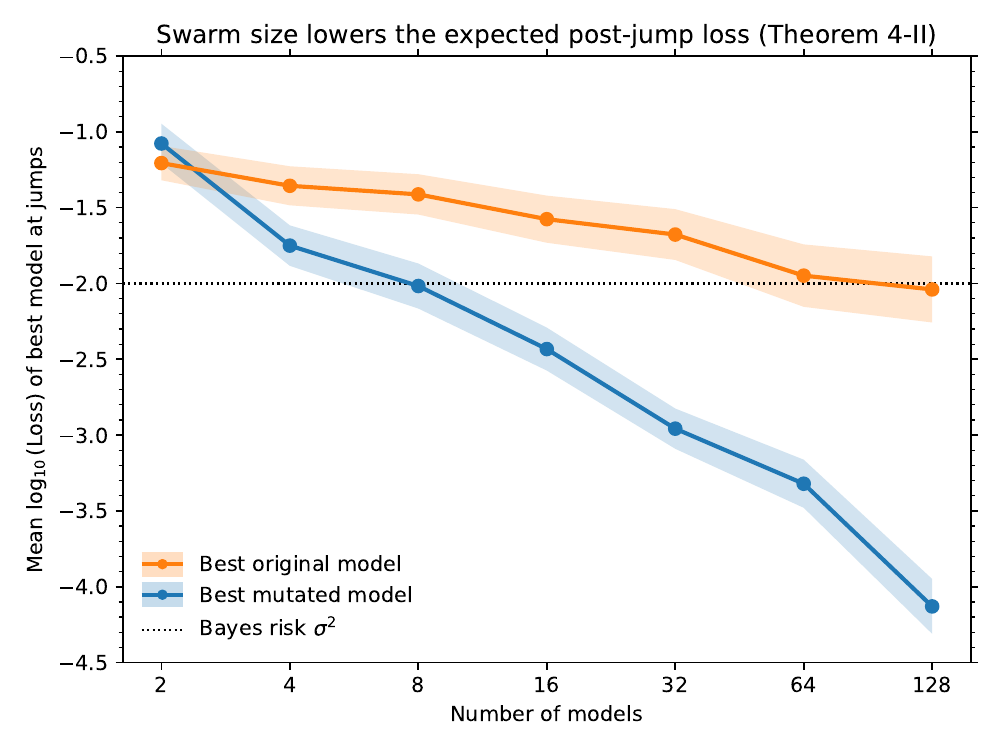}
  \caption{Experimental demonstration of Theorem 4 (II), i.e.\ the risk of a swarm of mutating models in a changing environment approaches zero as the number of models increases. Shown is the average MSE log-loss of the best model in the swarm at jump times as a function of the swarm's size, both for the mutated and the original swarm. The shaded region indicates the $1\sigma$ standard error on the mean across all jumps in each 1000-epoch run. The dotted line shows the Bayes risk of the environment, which is defined as the square of the noise in the target variable (see \autoref{fig:theorem1}). The swarm of mutating models quickly outperforms the Bayes risk and approaches zero as the number of models increases. This decrease is steeper than for the original swarm, which struggles to break the Bayes risk of the environment. This is in line with the assertion that the expected post-jump gap decreases with model diversity (see \autoref{fig:theorem4a}).}
  \label{fig:theorem4b}
\end{figure}
Finally, we consider the second part of \autoref{thm:swarm_gap}, i.e.\ that the risk of the swarm tends to zero as the number of models approaches infinity. To evaluate this theorem, we repeat the default setup and select the jump epochs. However, where previously we used a swarm of $N_{\mathrm{m}}=32$ models, we now vary the swarm size by integer powers of 2, from $N_{\mathrm{m}}=2$ to $N_{\mathrm{m}}=128$. For each swarm size, we calculate the mean log-loss of the best model in the swarm at jump times, as well as the standard error on the mean. The result is shown in \autoref{fig:theorem4b}. We see that both the best original model and the best mutated model achieve lower mean losses at jumps for larger swarms. For the original swarm, this is the result of sampling statistics, i.e.\ the minimum is expected to be lower if we draw from the parent distribution more often. For the mutated swarm, the loss decreases towards larger numbers of models for the additional reason that the mutations scatter some predictors closer to the Bayes predictor post-jump. Therefore, the decline of the best model's loss with $N_{\mathrm{m}}$ is steeper for the mutating swarm than for the original swarm.

The best model of either swarm can have an empirical mean loss below the nominal noise floor (Bayes risk). Each realized loss is the sum of the expected loss (its risk, which cannot be smaller than the Bayes risk for any individual model) and a noise term (which can push individual samples above or below the Bayes risk). By taking the minimum loss over all models at each jump and then averaging that minimum across epochs, we favor geometric proximity to the new Bayes predictor (which is more easily achieved by the mutated swarm, as explained above) and favorable noise realizations. As a result, the mean of the minimum loss can systematically lie below the Bayes risk, especially for large swarms. We indeed see this in \autoref{fig:theorem4b}. The mutated swarm breaks the Bayes risk of the environment for $N_{\mathrm{m}}\geq8$, whereas the best model in the original swarm requires 128 models to reach the Bayes risk. As the number of models increases, the loss of the best mutated model approaches zero. This confirms the limit $\rho_{N_{\mathrm{m}}}\downarrow 0$ as $N_{\mathrm{m}}\to\infty$ from \autoref{thm:swarm_gap}.

We observe that the standard error on the mean of the original swarm grows towards larger $N_{\mathrm{m}}$, which reflects sampling statistics. As $N_{\mathrm{m}}$ increases, improvements of the best original model require increasingly rare and extreme noise realizations, and therefore the jump-to-jump variability of the minimum loss increases. By contrast, the mutating swarm does not exhibit such an increase of the standard error with $N_{\mathrm{m}}$. This happens because the mutating swarm reduces the post-jump gap predominantly through systematic geometric effects (i.e.\ scattering some models closer to the new Bayes predictor in parameter space) rather than through rare, high-variance noise fluctuations.

\subsection{Inference Synthesis} \label{sec:results_inference_synthesis}
Given that the best model in the mutated swarm typically outperforms the best model in the original swarm, there exists an advantage to be captured by an inference synthesis mechanism. We now use the inference synthesis mechanism described in \S\ref{sec:inference} to synthesize the network inferences of the original and mutated model swarms, and quantify to what extent this mechanism is able to capture the Flawed-in-Nature advantage. To reiterate, the inference synthesis mechanism is a simple weighted average of the model inferences, with the weights calculated by passing the standardized EMA of the regret through a logistic gate. In the following, we apply this calculation to the model swarms considered in \S\ref{sec:results_swarm_statistics}.

\begin{figure}[!t]
  \centering
  \includegraphics[width=\textwidth]{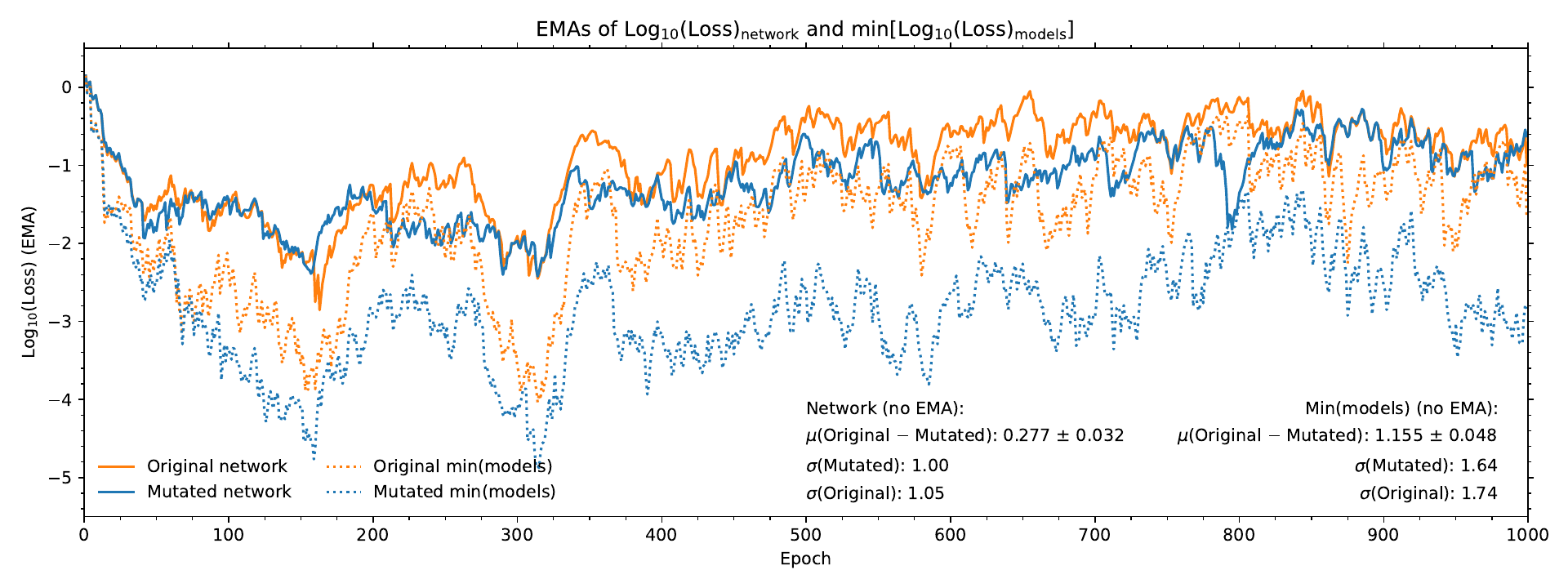}
  \caption{Exponential moving average (EMA) of the inference loss history for the original (orange) and mutated (blue) model swarms. The solid lines show the MSE log-loss for the synthesized network inference (calculated as in \S\ref{sec:inference}), and the dotted lines represent the MSE log-loss of the best model in each of the model swarms. As quantified in the statistics in the bottom right of the figure, the mutated swarm exhibits a smaller loss than the original swarm both for the network inference and the best model, illustrating that the adopted inference synthesis mechanism \citep{kruijssen24} is able to capture the Flawed-in-Nature advantage offered by the mutated swarm.}
  \label{fig:loss_history}
\end{figure}
\autoref{fig:loss_history} shows the EMA-smoothed log-loss trajectories of both network inferences and their best constituent models. After a burn-in period of about 200 epochs, the loss trajectories of the network inferences split and the mutated swarm systematically achieves a lower average loss than the original swarm. The mean log-loss difference is $0.277\pm0.032$, which is statistically significant at the $8.6\sigma$ level. The best models of the two swarms exhibit an even greater difference, which is unsurprising given that the mean log-loss improvement across the mutated swarm is negative (see \autoref{fig:loss_improvement}). As expected, the log-loss volatility of the network inferences is about 40\% lower than that of the best models, and the mutated swarm achieves marginally lower volatility (by about 5\%) than the original swarm. This experiment empirically illustrates that the inference synthesis mechanism of \citet{kruijssen24} captures the Flawed-in-Nature advantage, and even achieves a lower log-loss volatility despite the fact that the mutated swarm exhibits an elevated variance in model performance relative to the original swarm (see \S\ref{sec:results_swarm_statistics}).

\begin{figure}[!t]
  \centering
  \includegraphics[width=\textwidth]{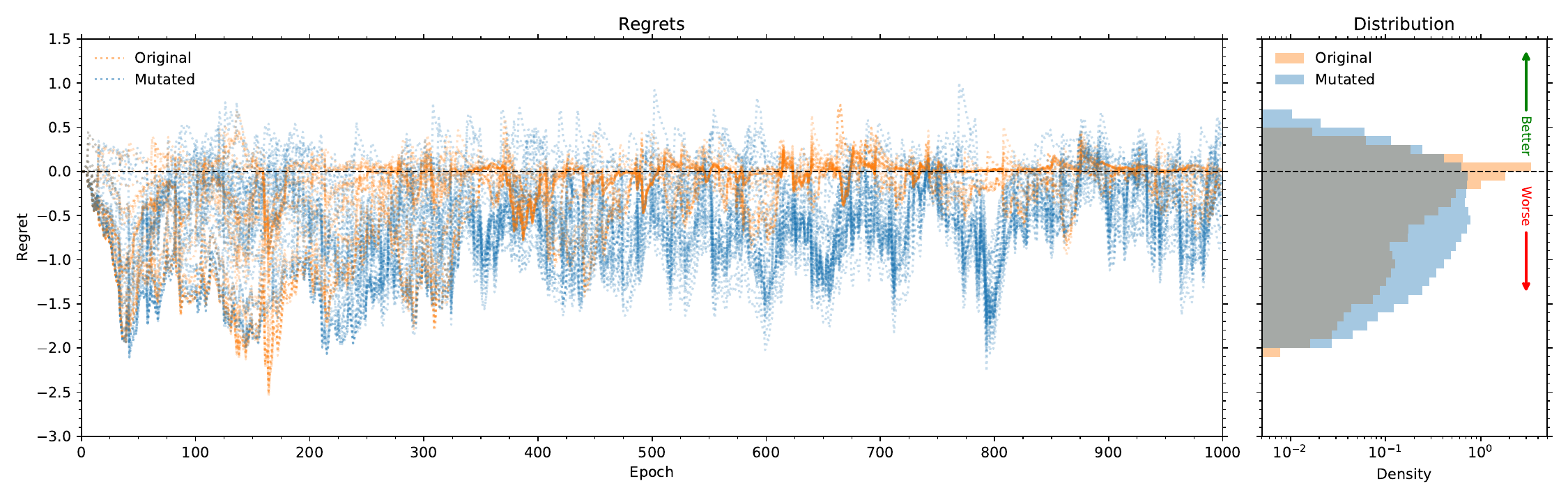}
  \caption{Regret evolution (left) and distribution (right) for the original and mutated model swarms. As in \autoref{fig:loss_history}, the orange lines show the original models, and the blue lines show the mutated models. Since the regret is defined from the perspective of the network inference (see \autoref{eq:regret}), models that outperform the network inference have positive regret. We see that the mutated swarm exhibits greater variance in regrets than the original swarm. While on average its regret is lower than that of the original swarm, it also hosts models with more extreme positive regrets, which are allocated greater weight by the inference synthesis mechanism through the logistic gate of \autoref{eq:g_x}.}
  \label{fig:regret_history}
\end{figure}
Having established the empirical outperformance of the network inference of the mutated swarm, it is worth investigating how this advantage originates. \autoref{fig:regret_history} shows the regret trajectories of the original and mutated model swarms, showing each individual model as a dotted line. Since the regret is defined from the perspective of the network inference (see \autoref{eq:regret}), a positive regret indicates that the model is outperforming the network inference. The histogram on the right-hand panel shows the time-integrated regret distribution across all models in both swarms, and demonstrates that the greater variance in log-loss improvement we identified in \autoref{fig:loss_improvement} also translates into a greater variance in regret. As a result, we find a lower average regret for the mutated swarm, but also more extremely positive outliers. These high-regret models receive disproportionate weight from the inference synthesis mechanism through the logistic gate of \autoref{eq:g_x}, thereby driving the network inference towards them. These outperforming models are visible as the many blue positive spikes in the left-hand panel.

In \S\ref{sec:advantage}, we demonstrated under which conditions the inference synthesis mechanism is able to capture the Flawed-in-Nature advantage, and specifically showed that the only way in which the synthesized inference can lag the best model is through the allocation penalty $(1-\hat{w}_{i\star})$, i.e.\ the weight not assigned to the best model. The post-jump behavior of the allocation penalty depends on the size of the gap between the best and second-best model. If the gap is large, the best model has a clear advantage and we can formulate an upper bound on the allocation penalty (\autoref{eq:norm_weight_monotone_2}). In this regime, the allocation penalty will exponentially decay towards a small value. This decay is not instantaneous, because the weights are derived from the standardized EMA of the regret. As shown in \S\ref{sec:advantage}, the decay should happen faster for steeper logistic gates (higher $p$), faster regret EMAs (higher $\alpha$), or a larger standardized post-jump advantage of the best model. If the gap is small, then the loss of the synthesized inference necessarily lies close to that of the best model, because the other models have a similar loss. In that case, the allocation penalty does not become small, but the synthesized inference would not benefit strongly from concentrating the weight on the best model.

\begin{figure}[!t]
  \centering
  \includegraphics[width=\textwidth]{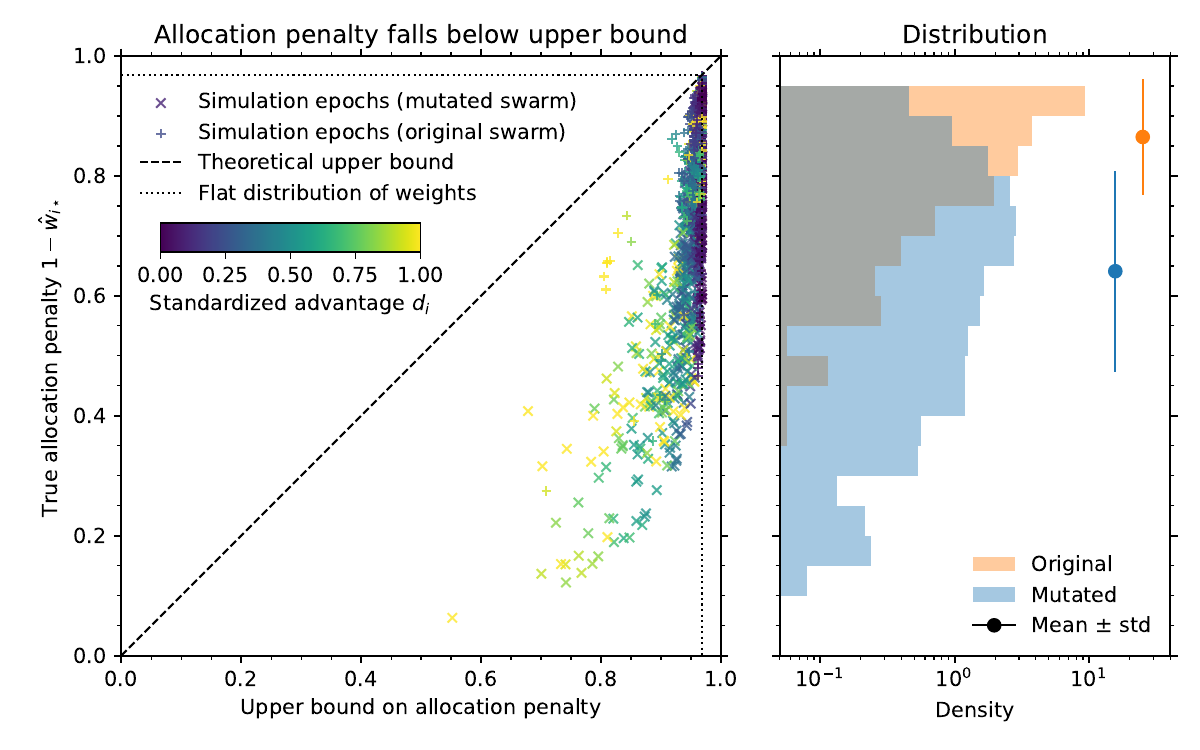}
  \caption{Experimental demonstration that the allocation penalty, i.e.\ the weight not assigned to the best model during inference synthesis, falls below the upper bound derived in \S\ref{sec:advantage}. Left: Allocation penalty at jump times as a function of its predicted upper bound, with colors denoting the standardized advantage of the best model relative to its closest competitor. The dashed line shows the upper bound and the dotted line shows the allocation penalty for a flat distribution of weights. We see that models with a greater advantage exhibit a smaller allocation penalty. Right: Histogram of the allocation penalties for the mutated and original swarms, with the data points indicating the mean and standard errors of both distributions. The allocation penalty falls below the upper bound for all jumps, indicating that the allocation penalty is appropriately bounded. The mutated swarm exhibits a smaller allocation penalty than the original swarm, which matches the result that the expected post-jump gap to the Bayes predictor decreases with model diversity (see \autoref{fig:theorem4a}).}
  \label{fig:allocation_penalty}
\end{figure}
We now evaluate the allocation penalty in our numerical experiment. We only consider the last 500 epochs to avoid any initialization effects during the first half of the experiment. \autoref{fig:allocation_penalty} demonstrates that, across the experiment, the actual allocation penalty is safely below the analytic upper bound from \autoref{eq:norm_weight_monotone_2}. The data points are colored by the standardized advantage, i.e.\ the gap between the best and second-best model, which illustrates that models with a greater advantage achieve smaller allocation penalties. In other words, the inference synthesis mechanism successfully identifies and concentrates its weight on these outperforming models. The distribution of allocation penalties in the right-hand panel shows that the mutated swarm achieves smaller allocation penalties on average than the original swarm. This is a direct consequence of the lower losses and higher regrets achieved by the best models in the mutated swarm, which results in a greater post-jump advantage. The left-hand panel confirms this further by showing a greater proportion of mutated models with high standardized advantage than for the original swarm.

In this experiment, mutation creates larger and more frequent gaps between models. \autoref{fig:allocation_penalty} shows that the inference synthesis gate exploits the larger model differentiation in the mutated swarm to concentrate more weight on its best model, whereas the original swarm behaves almost like a flat mixture. In that sense, mutation makes the synthesis mechanism operate more efficiently than it does on the original swarm.

We now test whether the half-life of the allocation penalty after a jump follows the theoretical estimate (\autoref{eq:decay-rate} and \autoref{eq:half-life-approx}). For each jump, we analyze the allocation penalty time series for the model favored by the logistic gate, i.e.\ the model with the largest standardized regret right after the jump. We then focus on a post-jump window that ends either when the next jump occurs or when a fixed horizon of 120 epochs is reached, so that the decay is measured within a single stationary segment rather than being contaminated by later jumps. Within each segment, we choose a reference epoch $t_0$ as the time at which the gate-best model's allocation penalty is maximal, i.e.\ it is clear which model is best, but the weight has not yet been concentrated. At that epoch, we calculate the observed separation $D$ as the log-loss gap between the gate-best model and its closest competitor at the same epoch, and we compute $\sigma$ as the standard deviation of the network regrets across all models from the preceding epoch. These quantities enter directly into the theoretical half-life prediction, together with the inference synthesis parameters $p$ (the gate's steepness) and $\alpha$ (the regret EMA parameter).

Empirically, we measure the half-life of the allocation penalty $(1-\hat{w}_{i\star})$ by following its evolution after each jump. The asymptotic floor is estimated from the median penalty in the final third of the observation window, and the target level is set halfway between the initial penalty and this floor. To ensure the threshold crossing is not distorted by upward fluctuations immediately preceding the crossing, we replace the raw penalty with its running minimum to obtain a monotonically non-increasing series. The empirical half-life is then taken to be the first time this monotonized trace falls below the target, with the exact crossing time obtained by linear interpolation between adjacent epochs. Finally, if the penalty never reaches the target before the next environmental jump, the half-life is set to the full window length rather than extrapolating beyond the available data.

The exponential decay analysis of \S\ref{sec:advantage} applies only to the regime where the post-jump advantage of the best model is large. To select the jumps in this regime, we require that (1) the gate-identified best model (i.e.\ by regret EMA) coincides with the model with the best instantaneous loss, (2) it exhibits a standardized log-loss advantage $D/\sigma >0.05$ relative to its competitors, (3) the initial allocation penalty exceeds its empirically-estimated floor (measured as the median of the final third of the segment) by at least 2\% (for a precision scale $\sim1/N_{\mathrm m}$) to ensure measurable decay, and (4) the segment contains at least 8 post-jump epochs. These criteria isolate clean episodes where the half-life estimation is meaningful.

Because only a small fraction of detected coefficient changes satisfy the criteria that isolate the clearly-best model and within-segment decay needed for a meaningful half-life estimate, we extend the simulation to 5000 epochs to accumulate enough qualifying events for both swarms to support the comparison. The original jump sample across 1000 epochs is reduced from 88 jumps to just 7 jumps, but for 5000 epochs it is reduced from 507 jumps to a meaningful 33 jumps for the mutated swarm, and 14 jumps for the original swarm. Applying the above criteria $(1,2,3,4)$ sequentially eliminates $(339,6,89,40)$ jumps for the mutated swarm, and $(249,170,57,17)$ for the original swarm. The most striking aspect of these cuts is that the original swarm rarely satisfies requirement (2), i.e.\ that the best model exhibits a sufficient standardized log-loss advantage. This is consistent with our persistent interpretation that mutations generate a greater swarm diversity.

\begin{figure}[!t]
  \centering
  \includegraphics[width=\textwidth]{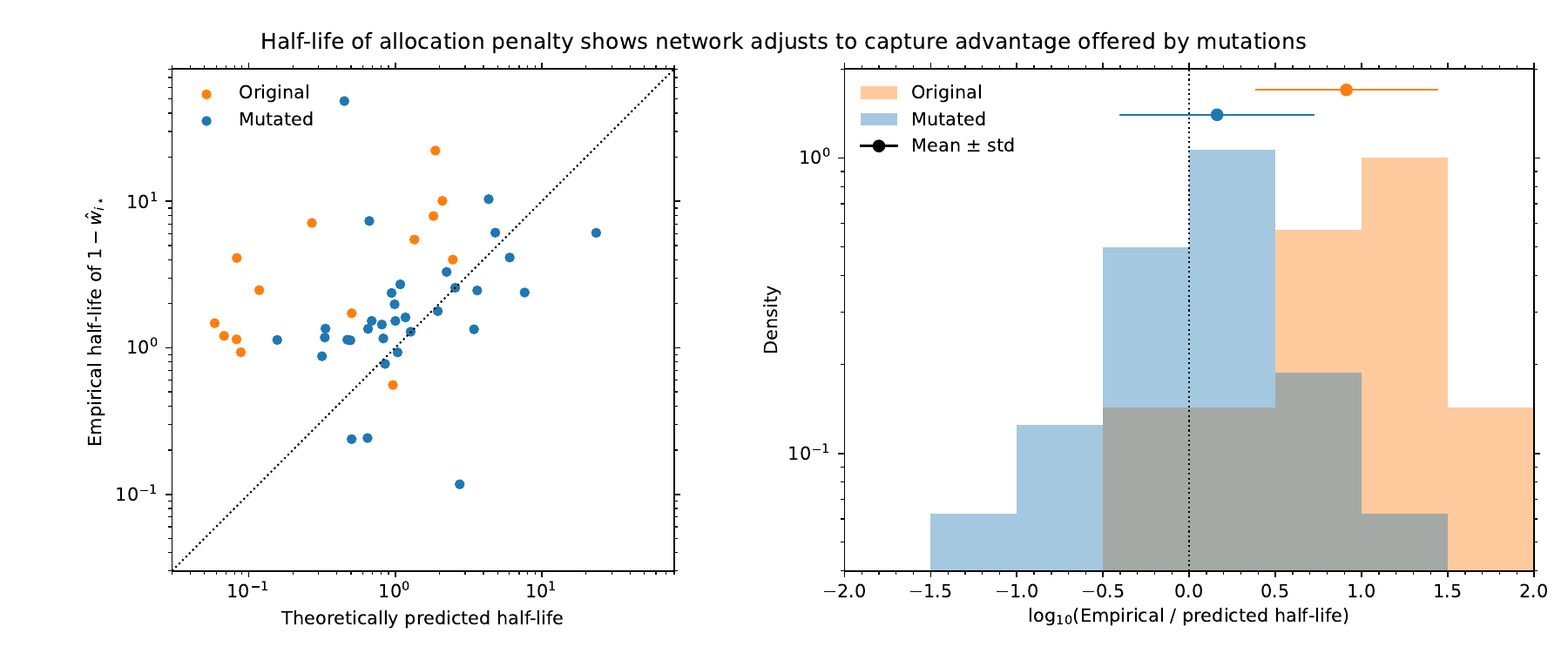}
  \caption{Experimental demonstration that the half-life of the allocation penalty qualitatively matches the theoretical expectation. Left: Half-life of the allocation penalty as a function of the theoretical expectation for the mutated (blue) and original (orange) swarms. The dotted line shows the 1:1 relation. Right: Histogram of the ratio between the measured and predicted half-lives for the mutated and original swarms, with the data points indicating the mean and standard deviations of both distributions. The mutated swarm has an allocation penalty that decays on a timescale roughly consistent with the theoretical expectation, whereas the original swarm shows no such agreement.}
  \label{fig:half_life}
\end{figure}
The remaining jumps that do not satisfy the filtering criteria are predominantly cases where no single model achieves a clear advantage after the jump, corresponding to Regime (B) in \S\ref{sec:advantage}. In this regime, models cluster in performance, so the worst gap $G_i^{\max}$ reflects the small spread among models rather than their absolute loss level. The excess loss bound in \autoref{eq:norm_weight_bound_3} then guarantees that the synthesized inference closely tracks the best available model, regardless of the weight distribution.\footnote{We note that this is a bound on the relative excess loss of the synthesis, not on its absolute performance.} The half-life analysis is therefore neither required nor appropriate for these jumps, as the synthesis mechanism's performance is guaranteed by the small gap rather than by rapid weight concentration.

\autoref{fig:half_life} shows the relation between the empirically-determined half-lives of the allocation penalty and the theoretically predicted values, as well as a histogram of the log-ratio between them, both for the original and the mutated model swarm. The figure empirically validates the theoretical half-life prediction from \autoref{eq:half-life-approx}, as the mutated swarm's measured half-lives cluster around the predicted values, whereas the half-lives of the original swarm show no such agreement and systematically exceed the predictions. Even among jumps with clear initial winners, the narrower diversity of the original swarm leads to less stable advantages that erode during the decay window, slowing weight convergence. This further affirms our conclusion from \autoref{fig:allocation_penalty} that model mutation increases the efficiency of the synthesis mechanism.

\begin{figure}[!t]
  \centering
  \includegraphics[width=\textwidth]{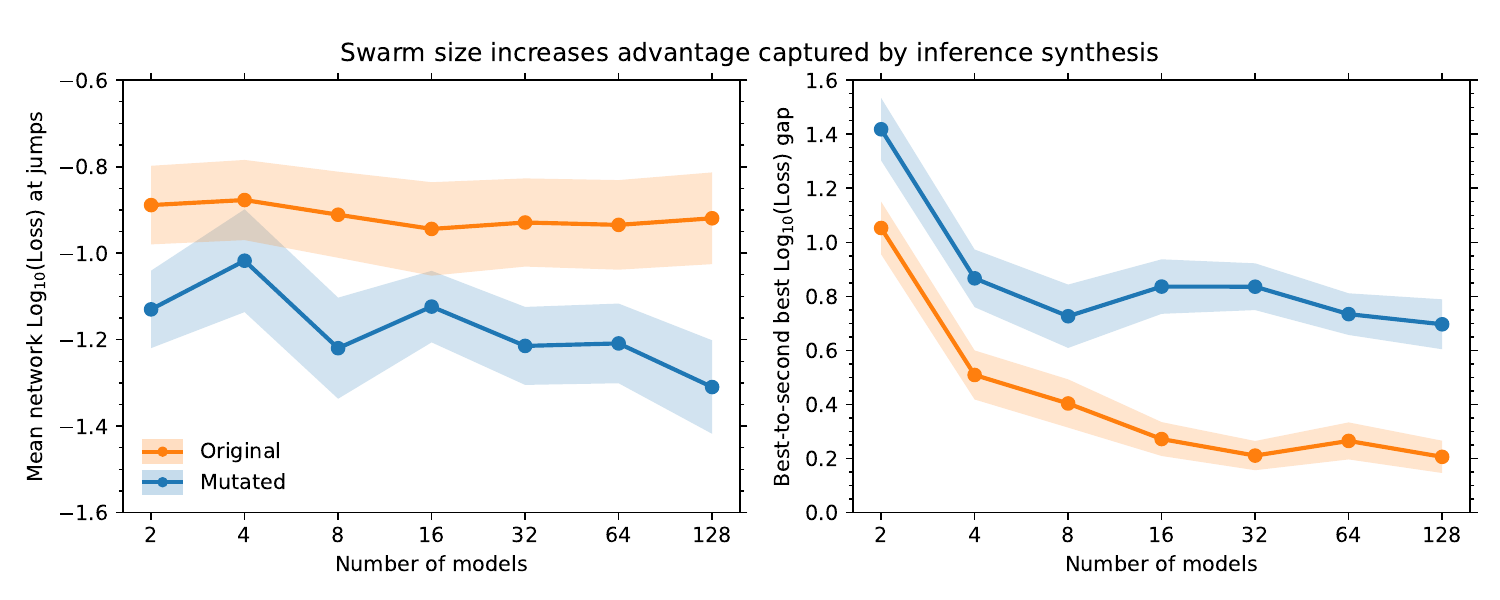}
  \caption{Experimental demonstration that the advantage captured by inference synthesis increases with increasing swarm size for the mutated swarm. Left: Repeat of \autoref{fig:theorem4b}, this time comparing the mutated and original network inference rather than the best models. Shown is the average MSE log-loss of the network inference at jump times as a function of the swarm's size. The mutated swarm outperforms the original swarm, and its loss decreases with increasing swarm size. Right: MSE log-loss gap between the best and second-best model as a function of the swarm's size. The mutated swarm exhibits a larger gap than the original, indicating a less crowded top of the loss distribution. In both panels, the shaded region indicates the $1\sigma$ standard error on the mean across all jumps in the 1000-epoch run. The figure shows that the mutated swarm samples a broader neighborhood in parameter space, and thus more frequently contains a dominant outlier after a jump, which provides a robust signal for inference synthesis to allocate most of the weight.}
  \label{fig:swarm_size}
\end{figure}
Finally, we revisit the scaling of the network's performance with the size of the swarm, which we first explored in \autoref{fig:theorem4b}. As before, we repeat the full setup over 1000 epochs and select the jump epochs, while varying the swarm size by integer powers of 2, from $N_{\mathrm m}=2$ to $N_{\mathrm m}=128$. For each of the two swarms, we calculate the mean network log-loss at jumps as well as its standard error. Additionally, we measure the log-loss gap between the best and second-best model at each jump.

\autoref{fig:swarm_size} shows both of these as a function of the number of models in the swarm. In the left-hand panel, we measure the accuracy of the network inference as opposed to the accuracy of the best model as in \autoref{fig:theorem4b}. As the number of models grows, the network's loss at jump times improves for the mutated swarm, whereas for the original swarm it remains roughly constant. This acts as a further demonstration that inference synthesis inherits the Flawed-in-Nature advantage, because it probes directly whether the synthesized inference itself benefits from increasing mutational diversity. Note that because the weights are computed from lagged regrets, the network loss at jump times reflects pre-shift weight allocations. Therefore, the observed advantage measures the extent to which inference synthesis was already concentrating weight on the mutated models that remain competitive immediately after a shift, rather than a post-shift reallocation response, which follows over the next few epochs and should amplify the advantage further, as shown in \autoref{fig:half_life}.

The right-hand panel confirms that the log-loss advantage of the mutated swarm from the left-hand panel is likely to amplify over subsequent epochs, because the winner margin $\overline{D}_i$ at jump epochs is greater for the mutated swarm than for the original swarm, regardless of the swarm size. When the advantage of the best model is greater, it is more cleanly separated from its nearest competitor, so that it is better identifiable and stable under noise. This way, the regret-based gate has a stronger and more stable signal to gravitate towards after the jump. The fact that $\overline{D}_i$ is greater for the mutated swarm indicates a less crowded bottom of the log-loss distribution (where we use $\overline{D}_i$ as a two-point proxy for that crowding). For both swarms, the best-to-second-best model advantage shrinks as the number of models increases, consistent with the increasing crowding that naturally follows from an increasing population density. However, the decline stagnates for the mutated swarm at $N_{\mathrm m}>4$, whereas for the original swarm it continues. The stagnation likely arises from the broader neighbourhood sampled in parameter space by the mutated swarm. Taken together, the trends shown in \autoref{fig:swarm_size} demonstrate that also the synthesized Flawed-in-Nature advantage increases with swarm size.

\subsection{Mutation Parameter Optimization}  \label{sec:optimization}
All experiments discussed so far are performed using the default environmental change and mutation parameters, i.e. $C_\mathrm{std} = M_\mathrm{std} = 0.1$ and $C_\mathrm{dt} = M_\mathrm{dt} = 10$. We now explore the effect of varying these parameters on the Flawed-in-Nature advantage, with the goal of identifying the optimal parameter set. To do so, we consider a grid of relative mutation timescales and magnitudes, each normalized to the environmental change timescale and magnitude, respectively. For each normalized parameter combination, we randomly generate the environmental change parameters from a log-uniform distribution in the ranges $\log_{10}{(C_\mathrm{std})} \in [-0.5, 0.5]$ and $\log_{10}{(C_\mathrm{dt})} \in [0.5, 1.5]$. The full grid is then generated to span 16 equidistant points across three orders of magnitude in each dimension, with $\log_{10}{(M_\mathrm{std}/C_\mathrm{std})} \in [-1.5, 1.5]$ and $\log_{10}{(M_\mathrm{dt}/C_\mathrm{dt})} \in [-1.5, 1.5]$. For each environmental change parameter combination, we run the default simulation of \S\ref{sec:results_swarm_statistics}, tracking two parallel swarms of 32 models each for 1000 epochs. In each case, we calculate the time-averaged log-loss improvement of the best model and the network inference.

\begin{figure}[!t]
  \centering
  \includegraphics[trim={1cm 0.5cm 2cm 1cm}, clip, width=0.497\textwidth]{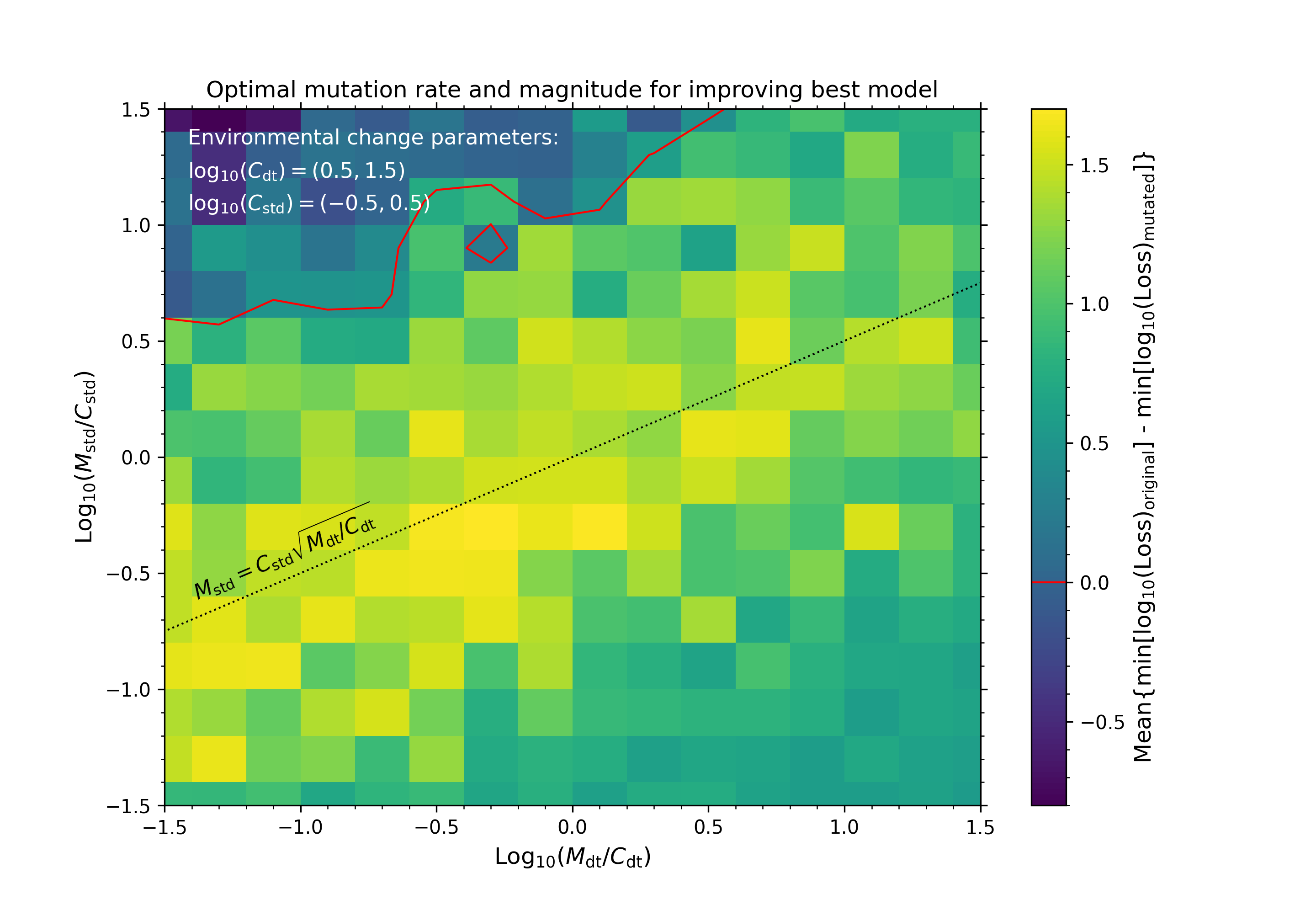}
  \includegraphics[trim={1cm 0.5cm 2cm 1cm}, clip, width=0.497\textwidth]{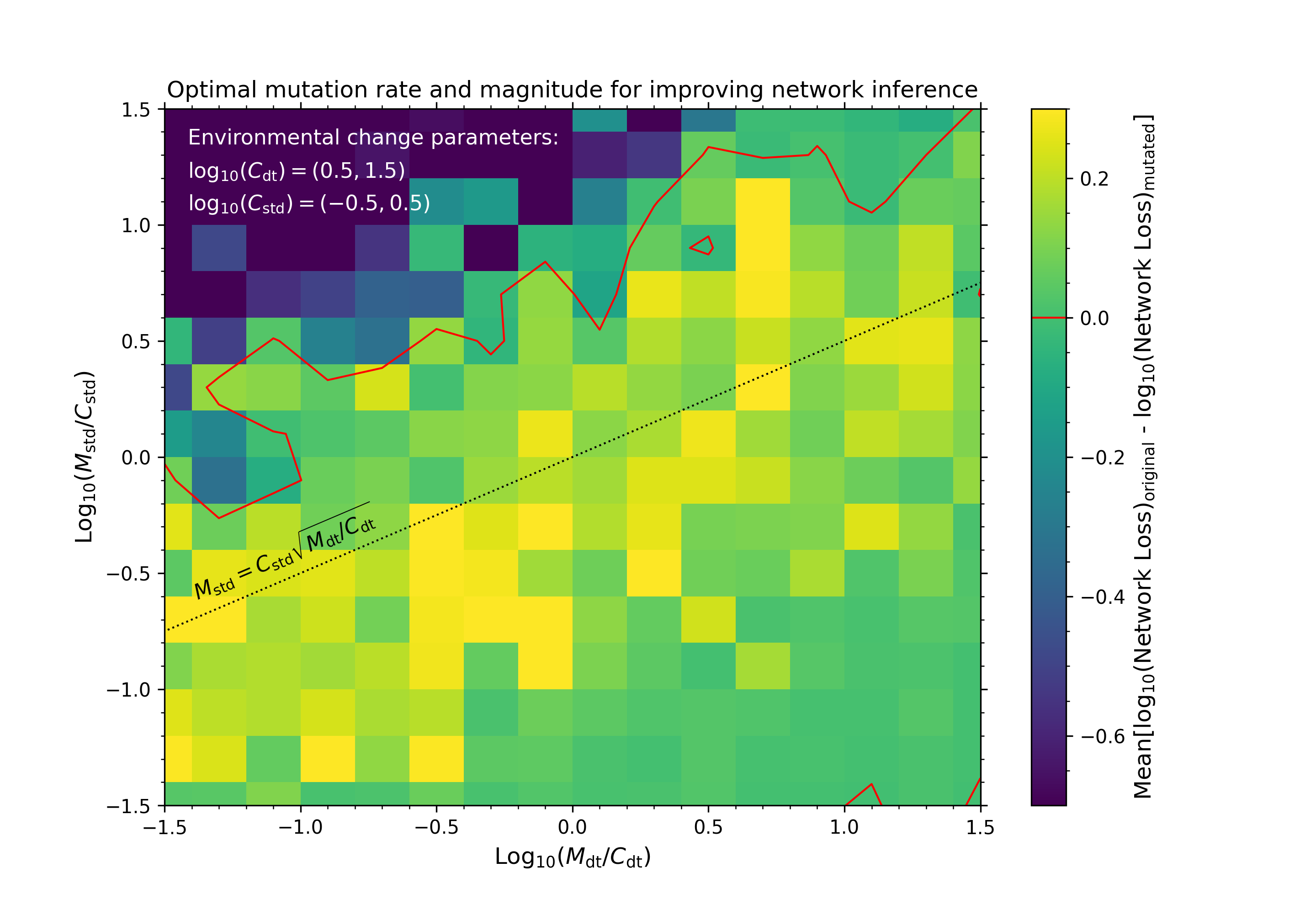}
  \caption{Optimal mutation rate and magnitude for the Flawed-in-Nature advantage. Shown are color maps of the mean log-loss improvement of the best model (left) and network inference (right) as a function of the relative mutation timescale ($x$-axis) and relative mutation magnitude ($y$-axis). The mutation parameters are relative, because they are scaled to the environmental coefficient change timescale and magnitude, respectively. The red contours indicate where the mutated and original swarms exhibit the same performance, i.e.\ where the Flawed-in-Nature advantage is zero. The black dotted line shows a naive theoretical expectation for a random walk process, where the mutation drift rate (proportional to $M_\mathrm{std}^2/M_\mathrm{dt}$) matches the environmental coefficient change rate (proportional to $C_\mathrm{std}^2/C_\mathrm{dt}$). This figure shows that the optimal mutation parameters where the Flawed-in-Nature advantage is maximized are indeed those where the mutated swarm mutates with a drift rate similar to the environmental evolution.}
  \label{fig:optimal_mutation_parameters}
\end{figure}
\autoref{fig:optimal_mutation_parameters} shows the mean log-loss improvement of the best model (left) and the network inference (right) as a function of the relative mutation timescale ($x$-axis) and relative mutation magnitude ($y$-axis), over three orders of magnitude in each dimension. We see that across most of the parameter space, the Flawed-in-Nature advantage is positive. Only when mutations are too large and too frequent does the mutation mechanism harm performance. This reflects a form of overshooting, where the mutated swarm mutates too quickly and too much, and therefore overshoots the new optimum.

The figure also shows a clear optimum along a rough linear relation. To understand the nature of this optimum, it is worthwhile to define the mutation drift rate $v_M$ and the environmental change rate $v_C$, i.e.
\begin{equation}
  v_M = \frac{M_\mathrm{std}^2}{M_\mathrm{dt}}, \quad v_C = \frac{C_\mathrm{std}^2}{C_\mathrm{dt}} .
\end{equation}
Each quantity reflects the usual scaling for the variance of displacement under a random walk after one time unit. The optimum in \autoref{fig:optimal_mutation_parameters} approximately follows a linear relation where $v_M = v_C$ (equivalent to $M_\mathrm{std} = C_\mathrm{std}\sqrt{M_\mathrm{dt}/C_\mathrm{dt}}$), i.e.\ the mutated swarm mutates with a rate similar to the environmental evolution. By matching the rate of drift, the best model can be expected to reside near the new optimum. Smaller mutation drift rates may move the best model in the right direction, but do so too slowly to reach the new optimum. Larger mutation drift rates may overshoot the new optimum, and therefore harm performance. The figure shows that the empirical optimum closely tracks this theoretical expectation.

Comparing both panels in \autoref{fig:optimal_mutation_parameters}, we see that the network inference benefits from similar optimal mutation parameters. This is consistent with the results from \S\ref{sec:results_inference_synthesis} showing that the network inference inherits the Flawed-in-Nature advantage from the best model. Due to the lagged response of the network inference, the advantage is not as pronounced as for the best model, but it is still present. For the network inference, the overshooting effect is more pronounced because the smaller advantage offers less room for overshooting. As a result, the optimal mutation parameters for the network inference favor a slightly slower mutation drift rate than the best model. This allows the network inference to avoid overshooting the new optimum.

\begin{figure}[!t]
  \centering
  \includegraphics[width=0.67\textwidth]{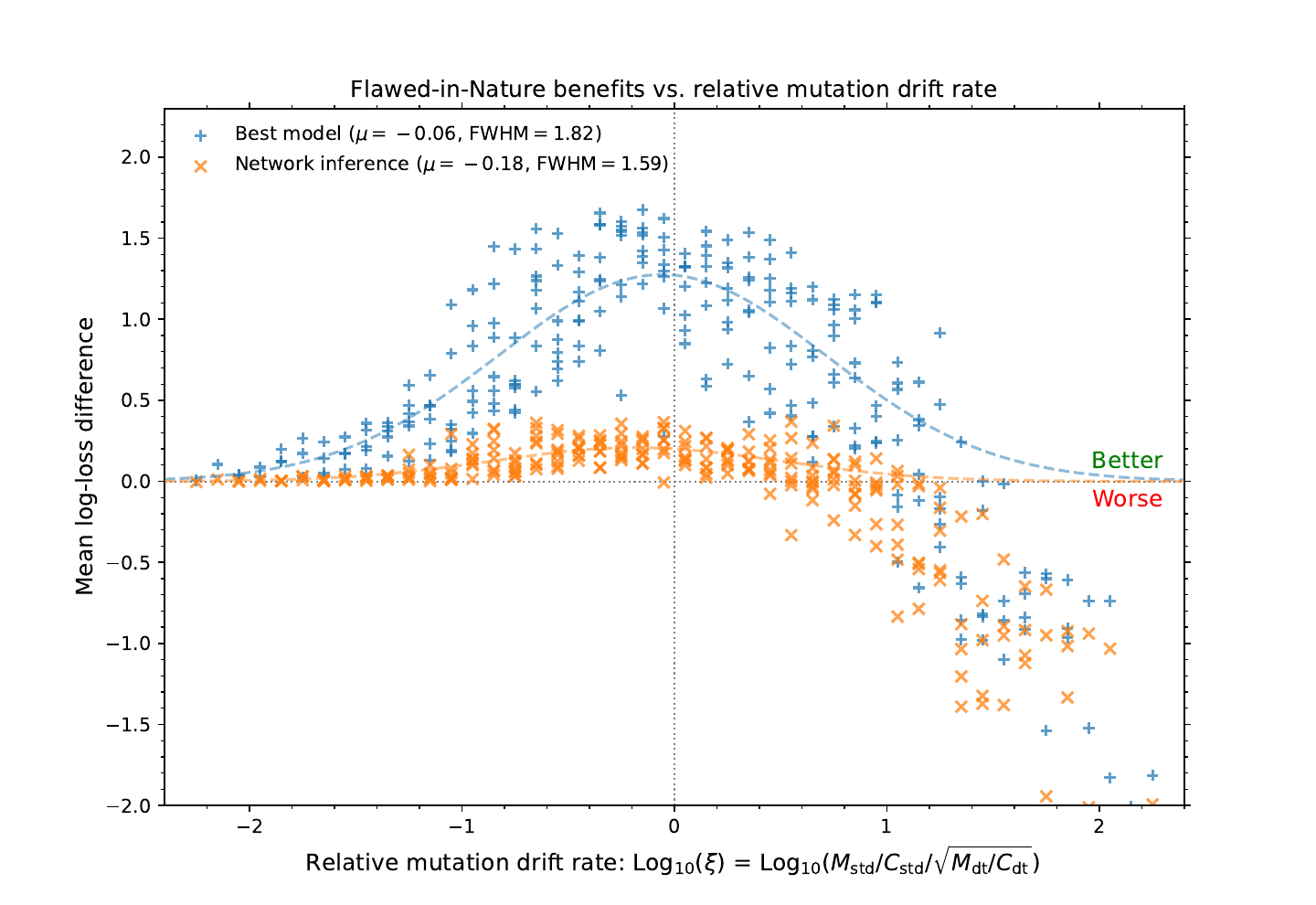}
  \caption{Mean log-loss improvement of the best model (blue plus symbols) and network inference (orange crosses) as a function of the logarithm of the relative mutation drift rate. Gaussian fits to the positive data points (dashed lines) are used to quantify the peak and width of the optimal regime, with best-fitting mean and full width at half maximum shown in the legend. The relative mutation drift rate $\xi$ is the square root of the ratio between the mutation drift rate and the environmental coefficient change rate, and expresses the distance orthogonal to the theoretical expectation for a random walk process. This one-dimensional collapse of the two-dimensional color maps shown in \autoref{fig:optimal_mutation_parameters} shows that the Flawed-in-Nature advantage is indeed maximized when the mutated swarm mutates with a drift rate similar to the environmental evolution, and only harms performance by mutating too quickly, but not by mutating too slowly.}
  \label{fig:flawed_in_nature_benefits}
\end{figure}
\autoref{fig:flawed_in_nature_benefits} projects the two-dimensional landscape of \autoref{fig:optimal_mutation_parameters} onto the one-dimensional relative mutation drift rate. The relative mutation drift rate is simply the square root of the ratio between the mutation drift rate and the environmental coefficient change rate:
\begin{equation}
  \xi = \sqrt{\frac{v_M}{v_C}} = \frac{M_\mathrm{std}}{C_\mathrm{std}}\sqrt{\frac{C_\mathrm{dt}}{M_\mathrm{dt}}} ,
\end{equation}
and expresses the distance orthogonal to the theoretical expectation for a random walk process in \autoref{fig:optimal_mutation_parameters}. The figure shows that the Flawed-in-Nature advantage is maximized near $\log_{10}(\xi) = 0$ (i.e.\ $v_M = v_C$), both for the best model and the network inference. Importantly, the benefit profile is asymmetric. The advantage plateaus for $\log_{10}(\xi) \ll 0$, showing that mutations that are too slow do not harm performance under the adopted, additive post-training mutation model, while it declines sharply and turns negative for $\log_{10}(\xi) \gtrsim 1$ through overshooting. This asymmetric benefit profile implies a conservative tuning strategy, wherein the model mutation drift should preferentially undershoot the environmental evolution. The optimal regime is broad, with a full width at half maximum of approximately $1.6$--$1.8$ orders of magnitude, suggesting robustness to parameter misspecification.

\section{Discussion} \label{sec:discussion}
We have presented the Flawed-in-Nature concept, wherein deliberate deviations from optimality across a population of models can yield collective outperformance. After providing a formal theoretical foundation in the form of mathematical proofs of four underlying theorems, we presented a simple synthesis mechanism that leverages the Flawed-in-Nature advantage to improve the performance of a network inference. We then validated the theoretical predictions and the synthesis mechanism through numerical experiments on a suite of synthetic linear regression problems. In this section, we discuss the limitations of these experiments, the practical applicability of the Flawed-in-Nature mechanism to realistic environments with unknown drift, and the impact of mutating AI/ML models on the future of swarm intelligence.

\subsection{Limitations of the Experiments} \label{sec:limitations}
The experiments presented in \S\ref{sec:experiments} are designed to isolate the core mechanisms that generate the Flawed-in-Nature advantage. This naturally implies some deliberate limitations on the scope of the experiments. In particular, the linear setting was chosen specifically because it enables a direct, interpretable comparison between the mutation drift rate and the environmental change rate, which is the central design parameter identified by the theory.
\begin{enumerate}
\item
The experiments are restricted to linear models, as all experiments are performed using ridge regression. While the theorems presented in \S\ref{sec:proof} and their proofs are stated in terms agnostic to the model class, the numerical validation is confined to a simple hypothesis class, and the degree to which coefficient-level mutations translate to productive exploration of the loss landscape in non-linear models (e.g.\ neural networks) remains an open empirical question. The restriction to linear models was made to enable the direct comparison between environmental changes and model coefficient mutations, and the extension to non-linear models is intentionally deferred to future work. To guide such extensions, we note that the linear setting possesses several properties that the experimental results depend on: (1) a linear mapping between coefficient perturbations and prediction changes, which gives the mutation drift rate a direct, interpretable relationship to the rate of change in the model's output; (2) a convex loss landscape without local minima, which ensures both that the optimal mutation parameters in \S\ref{sec:optimization} yield a clean optimum and that the synthesis bound in \S\ref{sec:advantage} holds; and (3) low dimensionality, which is discussed separately below. Of these, property (1) is the most consequential for non-linear extensions, and the sensitivity-scaling mechanism proposed in \S\ref{sec:practical} is specifically designed to address this point.
\item 
The adopted environment is defined by just four parameters. In higher-dimensional models, the volume of parameter space grows exponentially with the number of parameters, which means that the probability of finding a good model in a random search decreases exponentially. This would imply that exponentially more models would be required to cover the neighborhood of the new optimum after a jump. However, the loss landscapes of high-dimensional models are known in practice to exhibit low effective dimensionality, with gradients concentrating along a low-rank subspace \citep[e.g.][]{li18,gur-ari18}. Mutations that are scaled by the local gradient, as discussed in \S\ref{sec:practical}, would naturally concentrate perturbations along these sensitive directions, reducing the effective search space. Without such structure, the effective coverage of the mutation-induced spread (which is already limited at $p=4$ in our experiments, see the discussion following \autoref{thm:swarm_gap}) would degrade further with dimensionality.
\item
We have deliberately adopted a synthetic environment with known drift structure, where the environmental changes are driven by a Poisson process with a known rate and magnitude. Real-world environments generally deviate from this assumption and may exhibit continuous drift, correlated increments, or heterogeneous drift across dimensions, none of which are tested here. Again, the purpose of this simplification is to enable the direct comparison between environmental changes and model coefficient mutations.
\item
All models in the mutated swarm share the same mutation rate and magnitude. We control their values to compare the implied mutation drift rate to the environmental evolution rate. That is what allows us to calculate the relative mutation drift rate of each experiment. In practice, these parameters are unknown, and we experimented with setups where each model has its own mutation rate and magnitude. These experiments did not meaningfully improve the results, but greatly increased the variance among the mutated swarm. As a result, more models would be needed to achieve the same level of performance of the best mutated model. The uniform random sampling of mutation parameters across models is sample-inefficient, and the implied computational burden favors other approaches to handle the unknown optimal drift rate. In \S\ref{sec:practical}, we propose that this challenge can be addressed by a controller that dynamically optimizes the mutation parameters.
\item
The typical model swarm employed in the experiments consists of 32 models, and extends up to 128 models in the scaling experiments. This dynamic range is sufficient to highlight the scaling trends and satisfies reasonable computational constraints. However, we have not tested whether the Flawed-in-Nature benefit may potentially saturate in larger swarm sizes. \autoref{thm:swarm_gap} states that the limit on the expected post-jump gap is zero when the swarm size is infinite. While the theorem predicts continued improvement in the limit, the marginal benefit per additional model may diminish well before $N_{\mathrm m} \to \infty$, which represents a form of practical saturation we have not probed. Therefore, we acknowledge that saturation effects could be revealed in experiments extending to swarm sizes $N_{\mathrm m} \gg 128$.
\item
The experiments do not include dynamic model selection or pruning, only weight allocation through the inference synthesis mechanism. In practice, models that persistently underperform are likely to be removed from the swarm. In the presented experiments, we intentionally excluded this aspect, as it adds another layer of complexity and would likely increase performance, by removing models that are not contributing to the Flawed-in-Nature advantage. Nonetheless, there may exist a trade-off between weight dilution and model diversity, both of which would be reduced by pruning.
\end{enumerate}

\subsection{Practical Applicability to Environments with Unknown Drift} \label{sec:practical}
Realistic models and environments are likely to be more complex than the setup considered in our numerical experiments. Most importantly, knowledge of the dimensionality, rate, and magnitude of environmental changes is typically not available, yet \S\ref{sec:experiments} shows that the Flawed-in-Nature advantage requires matching mutation drift rate to environmental evolution. This is a general challenge for any machine intelligence approach that requires adaptation to unknown environments. Fortunately, the mutation dynamics discussed in this work enable a way to empirically match the mutation drift rate to the environmental evolution. We propose a simple controller that tunes mutation strength without estimating drift rate or magnitude explicitly. We sketch such a controller below, and defer its detailed implementation and validation to a companion paper.

In brief, the controller subdivides a single swarm into a small number of equal-sized drift rate groups, measures the Flawed-in-Nature utility, and uses the relation between the utility and the drift rate to generate empirically-motivated updates to the mutation drift rates of the groups. The advantage of this approach is that the updates of the mutation drift rate are guided by the local gradient of the utility function, are naturally subjected to adaptive temporal smoothing, and include protection against over-exploration. The iterative nature of this mechanism enables the controller to continuously align the mutation drift rate to the environmental evolution.

Practically, this controller may employ three equal-sized groups that operate at logarithmically-spaced drift rate levels ($v_{-}, v_{0}, v_{+}$), where $\log v_{\pm}=\log v_0 \pm \Delta \log v$. For each group, we can define the instantaneous utility by subtracting a (possibly softened) group-level minimum of the loss from the base loss of the original model. The utility is expected to be noisy, hence an EMA of the utility is used to obtain a stable estimate for each group. This results in a three-point relation between the utility and the drift rate, to which a quadratic function of the log-drift rate can be fitted. If there exists statistically significant curvature in the fit, a Newton-like step moves the mean drift rate of all groups towards the estimated utility optimum. If the curvature is insignificant, a gradient-ascent step follows the slope. The offset $\Delta \log v$ can be updated to adapt to the local curvature. If the groups enclose a utility optimum, the drift rate offset between the groups is decreased to zoom in on the optimum. Conversely, if the groups do not enclose a utility optimum, the drift rate offset between the groups is increased in order to cover a broader range of drift rates and discover the optimum more quickly. This simple controller specification conceptually addresses the unknown environmental evolution problem by ensuring that mutation diversity remains calibrated to exploit post-jump proximity to the new optimum.

The second key challenge in practical application of the Flawed-in-Nature mechanism is that the mapping of mutations to changes in the target variable (or loss landscape) is not as trivial as in the simple linear models used in the numerical experiments of this paper. If the model coefficients to be mutated are randomly selected (e.g.\ through a Bernoulli process), this can be addressed by calculating the empirical prediction sensitivity of the model output with respect to the mutated coefficient. The mutation magnitude can then be scaled by the square root of the parent group's drift rate (because $v$ parameterizes the target-space variance of the mutation-induced prediction change) and the inverse of the coefficient's sensitivity. This ensures that the mutation-induced variance of the target variable has a similar magnitude to the drift rate parameter.

This practical specification requires access to per-model prediction sensitivities and a baseline predictor. These are achievable when the model swarm is under centralized control. A swarm of decentralized models is unsuitable for controlled mutation without careful coordination, which would likely require a form of cryptographic guarantees to maintain model privacy. The complexity of that problem is far beyond the scope of this paper. Our aim here is only to establish a clear theoretical path to implementation. Further empirical validation of the controller, as well as the design of any possible decentralized coordination mechanisms, are deferred to future work.

\subsection{Impact of Mutating AI/ML Models on the Future of Swarm Intelligence} \label{sec:future}
The mechanisms described in this work depart from traditional ML optimization paradigms in a fundamental way. Traditional ML optimization focuses on individual models and needs to maximize the performance of each individual model under consideration. This naturally makes the optimized model vulnerable to environmental changes. The present paper demonstrates that the traditional optimization objective breaks down when considering the realistic scenario of a population of models under environmental changes.

The Flawed-in-Nature mechanism inverts the traditional optimization objective by deliberately introducing parameter diversity at the individual level to benefit the collective. The concrete result of this inversion is that each individual model is expected to perform worse, but the performance of the best model in the swarm (and therefore the collective) is expected to improve due to statistical hedging. The idea that deliberate deviations from optimality can yield collective outperformance mirrors biological evolution, where variation is the raw material for adaptation \citep{darwin59}. Indeed, the proven success of biological evolution demonstrates the enduring utility (i.e.\ evolutionary fitness) of statistical hedging.

This paper suggests a new design principle for multi-model systems. Rather than training a number of copies of an optimized model (e.g.\ using different random seeds), we demonstrate that it is beneficial to perturb the models themselves away from optimality and make them deliberately diverse. By synthesizing the outputs of the resulting model population, the Flawed-in-Nature benefit is retained. This way, model diversity becomes a primary design objective.

The idea that model diversity improves collective performance has roots in population-based training \citep{jaderberg17}, deep ensembles \citep{lakshminarayanan17}, Monte Carlo dropout \citep{gal16}, and model soups \citep{wortsman22}. However, these methods generate model diversity through hyperparameter variation, random initialization, transient inference-time noise, or independent fine-tuning runs, respectively. None involve deliberate, cumulative parameter perturbations. The Flawed-in-Nature mechanism is distinct in that it specifically calibrates mutations to track environmental drift, provides theoretical guarantees on regret reduction under non-stationarity, and preserves the full model population through soft, regret-weighted inference synthesis rather than hard selection or weight averaging. As discussed in \S\ref{sec:limitations}, model selection or pruning may amplify the Flawed-in-Nature advantage further.

Our findings have implications across a variety of domains. First, as AI inference moves towards decentralized architectures (e.g.\ federated learning, blockchain-based inference markets, or other distributed systems), the Flawed-in-Nature mechanism could offer a way to maintain collective adaptability. To accomplish this, the major challenges include adapting the controller mechanism from \S\ref{sec:practical} to trustless settings and designing incentives to reward contributions to the Flawed-in-Nature advantage rather than individual model performance. Specifically, decentralized inference networks such as Allora \citep{kruijssen24} already implement regret-weighted inference synthesis. The present work shows that layering a mutation mechanism on top of such networks can improve collective performance. Realizing this potential is an important future research direction.

Secondly, the Flawed-in-Nature mechanism changes the perspective on ensemble learning, both in traditional ML and in other AI domains such as reinforcement learning and large language models (LLMs). In conventional ensemble methods, diversity may arise randomly from variation in training data, random seeds, or hyperparameters, but the individual models are each independently optimized. The Flawed-in-Nature mechanism adds deliberate post-training parameter perturbation as an independent source of model diversity that is particularly effective under environmental changes. In reinforcement learning, the Flawed-in-Nature mechanism extends population-based methods by mutating model parameters directly rather than hyperparameters. It provides a theoretical basis for why such population diversity is beneficial in non-stationary environments, and identifies the matching of mutation drift rate to environmental drift as the key design parameter.

The insight may also extend to domains currently dominated by monolithic models, such as LLMs. In LLMs, the environment is inherently non-stationary (e.g.\ due to evolving language, emerging knowledge, or shifting user preferences), but retraining is prohibitively expensive. Parameter-efficient fine-tuning methods (e.g.\ LoRA adapters, see \citealt{hu22}) could act as a low-dimensional mutation layer on top of a shared base model, making it computationally feasible to maintain a swarm of deliberately diversified model variants. Combined with inference synthesis, such an architecture could enable dynamic adaptation to environmental changes without retraining the base model. More broadly, the common practice of concentrating resources into a single, maximally-optimized model represents exactly the type of vulnerability that is highlighted by \autoref{thm:linear_regret}. It is an open empirical question whether LoRA's low-rank constraint provides sufficient coverage of the relevant parameter directions for the Flawed-in-Nature mechanism to be effective, or whether it concentrates mutations too narrowly. The answer will depend on the alignment between the adapter subspace and the directions of the environmental drift.

The Flawed-in-Nature mechanism is also likely complementary to continual learning. Continual learning methods \citep[e.g.][]{kirkpatrick17,zenke17} protect important parameters from being overwritten during retraining, effectively partitioning the parameter space into consolidated (protected) and plastic (modifiable) directions. Mutations applied to consolidated parameters would be corrected at the next training step, while mutations applied to plastic parameters would persist. Rather than being a limitation, this interaction could focus mutations on the directions where exploratory diversity is the least costly to the model's existing knowledge. However, if environmental drift affects the consolidated parameters, this channeling would naturally be counterproductive, as the mutations would explore where the model is flexible rather than where adaptation is needed. The balance between these effects likely depends on the alignment between the consolidation structure and the environmental drift, and characterizing this interaction quantitatively is an important direction for future work.

Accomplishing the above goals is contingent on addressing several immediate open questions. Even when applying the controller setup from \S\ref{sec:practical}, how should mutation policies depend on the model architecture? What should reasonable model elimination criteria look like? What incentive structures might be created to encourage mutations despite the resulting individual performance cost across a swarm of models? The broad scope of these questions implies a rich future research agenda.

\section{Conclusion: Flawed in Nature, Perfect through Evolution} \label{sec:conclusion}
In this paper, we demonstrate a new design principle for multi-model systems, wherein deliberate deviations from optimality are shown to yield collective outperformance when the environment experiences unpredictable drift. To highlight the counterintuitive nature of this principle, we refer to it as \textit{Flawed in Nature, Perfect through Evolution}. It is rooted in biological evolution, where intelligence has arisen through the combination of heritable variation (expressed as model mutation in \S\ref{sec:mutation_def}) and natural selection \citep{darwin59}. We follow an analogous approach. While forms of natural selection have long existed in multi-model learning frameworks, to our knowledge this is the first time a modeling technique introduces deliberate, cumulative parameter mutations with theoretical guarantees on regret reduction under non-stationarity. In nature, mutations are thought to have played a critical role in enabling evolutionary adaptation and diversification. Analogously, we believe the Flawed-in-Nature mechanism could pave the way for organically achieving more intelligent systems.

The main conclusions of this work are as follows.
\begin{enumerate}
  \item We present four theorems that establish the Flawed-in-Nature mechanism. The argument starts from the observation that a single model in a static environment achieves the minimum risk within   
  its hypothesis class, i.e.\ mutations do not offer any advantage (\autoref{thm:erm_consistency}). However, when the environment experiences unpredictable drift, the single model incurs linear regret and thereby inevitably reaches obsolescence (\autoref{thm:linear_regret}). This regret is information-theoretically irreducible, since a single model is still conditionally optimal given its access to past observations only (\autoref{thm:info_optimality}). Only a swarm of mutated models, which are deliberately perturbed away from information-theoretical optimality, is shown to break this bound by reducing the expected post-jump excess loss through statistical hedging (\autoref{thm:swarm_gap}). (\S\ref{sec:proof})
  \item The Flawed-in-Nature advantage can be captured through a simple synthesis mechanism that aggregates the outputs of the model swarm into a single inference. There exist many aggregation methods, and we adopt a simple regret-weighted linear pool, which is shown to rapidly concentrate weight on the best-performing mutated model after an environmental jump, enabling the synthesized inference to inherit the swarm's diversity benefit. (\S\ref{sec:inference})
  \item We validate the theoretical predictions through numerical experiments on synthetic linear regression problems with Poisson-driven coefficient drift. The experiments empirically confirm all four theorems. The best model in the mutated model swarm outperforms the best original model in $\sim80\%$ of the environmental jumps. The adopted inference synthesis mechanism successfully translates the Flawed-in-Nature advantage into a collective benefit, as the mutated swarm outperforms the original swarm both in terms of the best model and the synthesized inference. This advantage increases with the number of models in the swarm. Finally, we establish that the optimal mutation parameters are those that match the environmental drift rate under a random walk process. A lower mutation drift rate does not harm performance, but too high a mutation drift rate does through overshooting. (\S\ref{sec:experiments})
  \item We outline a simple, adaptive controller that learns the optimal mutation strength online without needing to estimate the environmental drift rate or magnitude directly. The controller subdivides a single swarm into a small number of equal-sized drift rate groups, measures the Flawed-in-Nature utility of each group, and uses the relation between the utility and the drift rate to generate empirically-motivated updates to the mutation drift rates. This simple controller specification conceptually addresses the problem of unknown environmental evolution, by ensuring that the generation rate of mutation diversity can be optimally calibrated to achieve post-jump proximity to the new optimum. This controller concept establishes a clear theoretical path to the real-world deployment of the Flawed-in-Nature mechanism. (\S\ref{sec:practical})
\end{enumerate}

Introducing random variation into model parameters worsens the expected accuracy of any individual model, but benefits the swarm's collective performance. The idea that individual imperfection can yield collective perfection is the essence of the Flawed-in-Nature principle, and it offers a new perspective on swarm intelligence and decentralized AI systems. The use of random mutation as a form of experimentation in the face of an uncertain future closely resembles the evolutionary path through which biological intelligence likely arose. After all, when an organism achieves high fitness under natural selection, the mutation that granted it that fitness took place long before the selection event. Nature too performs statistical hedging.

Current discussions in AI research often revolve around the question which ingredients are still lacking in frontier AI systems. We suggest that it is not a specific ingredient that has been missing, but a process. Based on its close analogy to biological evolution, it is not unthinkable that the Flawed-in-Nature mechanism could enable the organic discovery of forms of machine intelligence that more closely mimic biological intelligence. We speculate that the future of AI will be dominated by swarms of countless, independent models rather than monoliths. Many of these models will be specialized, others will be general. They will mutate and evolve, giving rise to emergent lineages of models. The evolutionary mechanisms that led to biological intelligence may well extend to technology. And if they do, the short adaptation cycles of AI models will accelerate the process far beyond biological standards.

\section*{Acknowledgments}
JMDK thanks Bryn Bellomy, M\'elanie Chevance, Nick Emmons, Steve Longmore, Kenny Peluso, Joel Pfeffer, and Florian Stecker for helpful discussions and feedback.

\begin{sloppypar}
\bibliographystyle{ADI}
{\small
\bibliography{ourbib}
}
\end{sloppypar}

\end{document}